%% file: main.tex
\documentclass[lettersize,journal]{IEEEtran}
\IEEEoverridecommandlockouts
\usepackage{amsmath,amssymb,amsfonts}
\usepackage{algorithmic}
\usepackage{graphicx}
\usepackage{textcomp}
\usepackage{comment}

\usepackage[T1]{fontenc}
\usepackage[utf8]{inputenc}
\usepackage{tgtermes}
\usepackage{multirow}
\usepackage{stfloats}

\usepackage[ruled,vlined,linesnumbered]{algorithm2e}
\usepackage{subcaption}
\usepackage{float}
\usepackage{booktabs}

\usepackage{xfrac}
\allowdisplaybreaks

\usepackage{caption} 
\usepackage[letterspace=-5]{microtype}

\usepackage[shortlabels]{enumitem}

\usepackage{array}
\newcommand{\PreserveBackslash}[1]{\let\temp=\\#1\let\\=\temp}
\newcolumntype{C}[1]{>{\PreserveBackslash\centering}p{#1}}
\newcolumntype{R}[1]{>{\PreserveBackslash\raggedleft}p{#1}}
\newcolumntype{L}[1]{>{\PreserveBackslash\raggedright}p{#1}}

\usepackage{amsthm}

\newtheorem{assumption}{Assumption}
\newtheorem{theorem}{Theorem}
\newtheorem{corollary}{Corollary}
\newtheorem{lemma}{Lemma}[section]

\usepackage[dvipsnames]{xcolor}

\newcommand{\Identity}{{\rm I\kern-.2em l}}
\newcommand{\Expect}{\mathbb{E}}
\newcommand{\Expectbracket}[1]{\mathbb{E}\left[ #1 \right]}
\newcommand{\Expectt}[1]{\mathbb{E}_t\left[ #1 \right]}
\newcommand{\Expectcond}[2]{\mathbb{E}\left[\left. #1 \right| #2 \right]}
\newcommand{\x}{\mathbf{x}}
\newcommand{\y}{\mathbf{y}}
\newcommand{\g}{\mathbf{g}}

\newcommand{\z}{\mathbf{z}}
\newcommand{\qmin}{q_\text{min}}

\newcommand{\normsq}[1]{\left\Vert #1 \right\Vert^2}
\newcommand{\innerprod}[1]{\left\langle #1 \right\rangle}
\newcommand{\Pmax}{P_\text{max}}
\newcommand{\iid}{i.i.d.}
\newcommand{\noniid}{non-i.i.d.}

\def\BibTeX{{\rm B\kern-.05em{\sc i\kern-.025em b}\kern-.08em
    T\kern-.1667em\lower.7ex\hbox{E}\kern-.125emX}}
\begin{document}

\title{Communication-Efficient Device Scheduling for Federated Learning Using Lyapunov Optimization}

\author{Jake B. Perazzone, Shiqiang Wang,~\IEEEmembership{Senior Member,~IEEE}, Mingyue Ji,~\IEEEmembership{Member,~IEEE}, \\  Kevin Chan,~\IEEEmembership{Senior Member,~IEEE}
\thanks{This research was partly sponsored by the U.S. Army Research Laboratory and the U.K. Ministry of Defence under Agreement Number W911NF-16-3-0001. The views and conclusions contained in this document are those of the authors and should not be interpreted as representing the official policies, either expressed or implied, of the U.S. Army Research Laboratory, the U.S. Government, the U.K. Ministry of Defence or the U.K. Government. The U.S. and U.K. Governments are authorized to reproduce and distribute reprints for Government purposes notwithstanding any copyright notation hereon.
The work of Mingyue Ji was supported in part by NSF CAREER Award 2145835 and NSF Award 2312227.
A preliminary version of this paper titled ``Communication-Efficient Device Scheduling for Federated Learning Using Stochastic Optimization'' was presented in the IEEE International Conference on Computer Communications (INFOCOM), 2022 \cite{perazzone2022communication}.}%
\thanks{Jake Perazzone and Kevin Chan are with DEVCOM Army Research Laboratory, Adelphi, MD 20783, USA (e-mails: \{jake.b.perazzone.civ; kevin.s.chan.civ\}@army.mil)}%
\thanks{Shiqiang Wang is with IBM T. J. Watson Research Center, Yorktown Heights, NY 10598, USA (e-mail: shiqiang.wang@ieee.org)}%
\thanks{Mingyue Ji is with the Department of Electrical and Computer Engineering, University of Florida, Gainesville, FL 32611, USA (e-mail: mingyue.ji@ufl.edu).}}

\maketitle

\begin{abstract}
Federated learning (FL) is a useful tool that enables the training of machine learning models over distributed data without having to collect data centrally.
When deploying FL in constrained wireless environments, however,  intermittent connectivity of devices, heterogeneous connection quality, and non-i.i.d. data can severely slow convergence.
In this paper, we consider FL with arbitrary device participation probabilities for each round and show that by weighing each device's update by the reciprocal of their per-round participation probability, we can guarantee convergence to a stationary point.
Our bound applies to non-convex loss functions and non-i.i.d. datasets and recovers state-of-the-art convergence rates for both full and uniform partial participation, including linear speedup, with only a single-sided learning rate.
Then, using the derived convergence bound, we develop a new online client selection and power allocation algorithm that utilizes the Lyapunov drift-plus-penalty framework to opportunistically minimize a function of the convergence bound and the average communication time under a transmit power constraint.
We use optimization over manifold techniques to obtain a solution to the minimization problem.
Thanks to the Lyapunov framework, one key feature of the algorithm is that knowledge of the channel distribution is not required and only the instantaneous channel state information needs to be known.
Using the CIFAR-10 dataset with varying levels of data heterogeneity, we show through simulations that the communication time can be significantly decreased using our algorithm compared to uniformly random participation, especially for heterogeneous channel conditions.
\end{abstract}


\input{Introduction}

\input{RelatedWorks}
\input{Problem}
\input{Algorithm}

\input{Experiments}
\input{Conclusion}

\input{Appendix}

\bibliographystyle{IEEEtran}
\bibliography{references.bib}

\end{document}

%% file: Introduction.tex
\section{Introduction}
Federated learning (FL) enables the training of machine learning (ML) models over decentralized data.
Instead of transmitting data from all devices to a centralized location, model training is accomplished through a collaborative procedure in which the participants train on their own locally collected datasets and periodically share their model parameters.
This ML technique is immensely powerful for protecting the privacy of the users' data since raw data never leaves the devices.
Furthermore, two other key features of FL, partial device participation and multiple local iterations, help reduce the communication burden by communicating less information less often compared to centralized learning or other distributed training techniques. 
Accordingly, FL is most advantageous in instances where full data transmission is infeasible or otherwise restricted in some way.

In this paper, we consider the scenario where training is coordinated by a central aggregator that communicates with each device over a wireless network.
The aggregator is responsible for choosing devices, accumulating their models, and disseminating the aggregated global model back to the devices.
We use device and node interchangeably throughout the paper.
A block diagram of the wireless network running FL can be found in Fig. \ref{fig:blockDiagram}.
Each node $n$ is a device that has a unique dataset and an uplink channel to the aggregator with channel gain $h_n^t$ at time $t$.
The orchestration of FL over large-scale wireless networks such as this is a challenging task due to the dynamics of the channel.
In mobile edge computing (MEC) environments, for example, poor channel quality and intermittent connectivity can completely derail training. 
In the original FL algorithm, \emph{FedAvg}~\cite{mcmahan2017communication}, devices are selected uniformly at random in each round.
Although this strategy has been shown to converge \cite{li2019convergence,karimireddy2020scaffold,mitra2021linear}, in practice, it may lead to poor performance since it is completely agnostic to the communication medium as well as many other practical factors.
For example, if devices with poor channels are naively chosen, the model will converge much slower in terms of wall-clock time since it will take longer to communicate the parameters in each aggregation round.
On the other hand, if such devices are ignored, convergence will also be negatively affected since those devices' data will not be used.
The end result is the consumption of more network resources than is necessary due to longer training times. 
Thus, a more intelligent approach to device selection is needed to optimize network resource consumption and minimize wall-clock convergence time.

\begin{figure}[h]
    \centering
    \includegraphics[width=.7\linewidth]{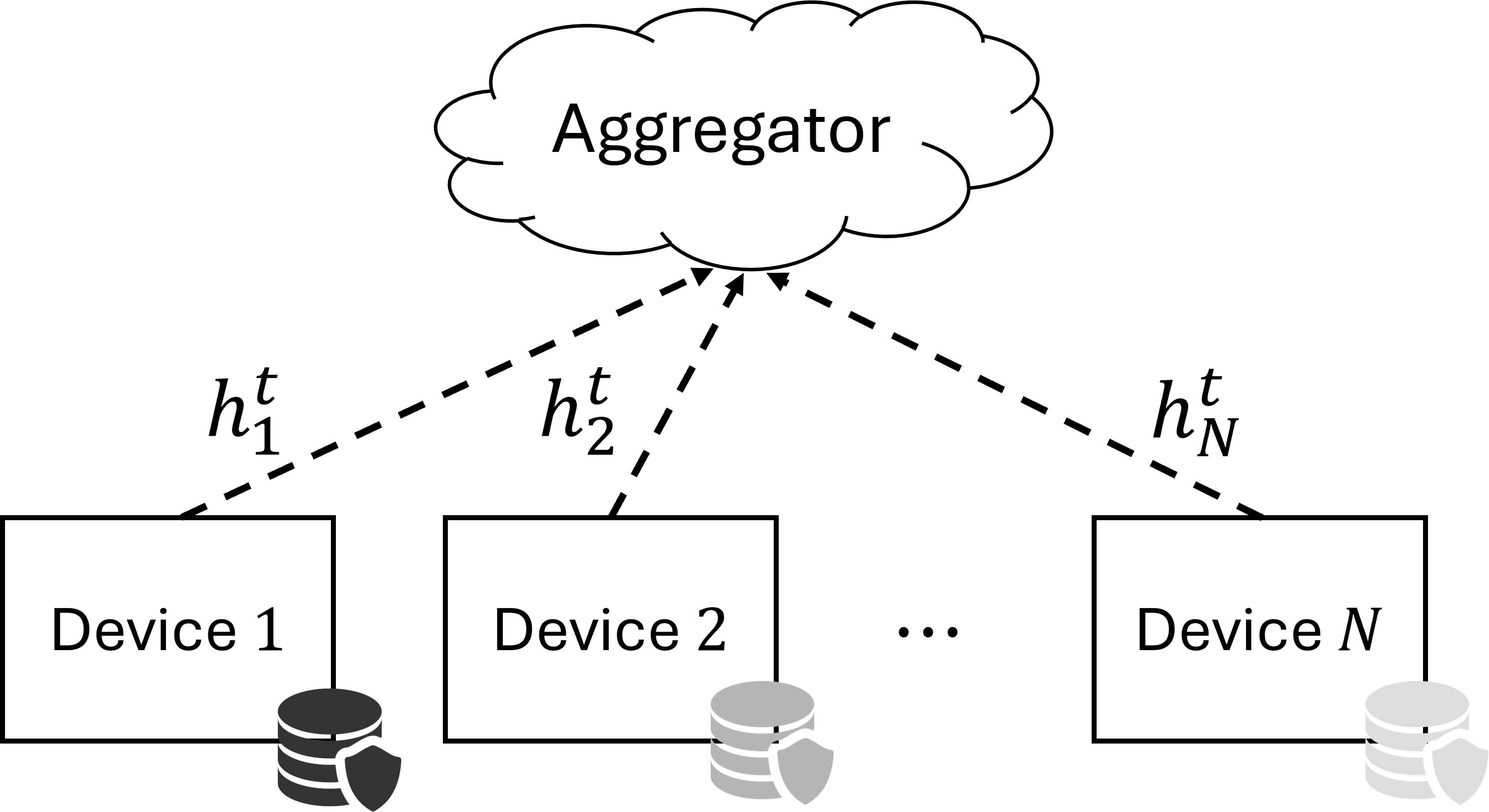}
    \caption{Block diagram of the uplink communication in federated learning over a wireless network.}
    \label{fig:blockDiagram}
\end{figure}

Before designing such an approach, though, the effect of arbitrary device selection probabilities on convergence must be understood to ensure convergence to a good model.
Therefore, in this paper, we first derive a convergence bound for non-convex loss functions with arbitrary device selection probabilities for each round and do not assume uniform bounded gradients as is sometimes done, e.g., \cite{perazzone2022communication, li2019convergence, li2018federated, yu2019parallel}.
In fact, we show that we can achieve linear speedup \cite{yang2021achieving} with only a single-sided learning rate.
We then use that bound to inform the development of a device selection and power allocation policy that greatly speeds up wall-clock convergence time.
More specifically, our upper bound shows that as long as all devices have a non-zero probability of participating in each round, then FL will converge in expectation to a stationary point of the loss function.
We then use the knowledge of how the selection probabilities affect the convergence bound to formulate a stochastic optimization problem that determines the optimal transmit powers and selection probabilities for a nonuniform sampling with replacement selection strategy.
Optimality in this case means it minimizes a weighted sum of the convergence bound and the time spent communicating model parameters while satisfying constraints on the peak and time average transmit power of each device.
The form of the convergence bound and our novel problem formulation allow us to utilize the Lyapunov drift-plus-penalty framework to solve the problem in an online and greedy fashion with optimality and constraint satisfaction guarantees. 
A key advantage of our new device selection algorithm is that it is able to make decisions according to current channel conditions without knowledge of the underlying channel statistics due to the structure of the Lyapunov framework.

To show the performance of our algorithm, we run numerous experiments on the CIFAR-10 dataset with varying levels of data and channel heterogeneity to demonstrate the saved training time using our developed algorithm.
We compare our results to the uniform selection policy of \emph{FedAvg} and show that the time required to reach a target accuracy can be sped up by up to $8.5$x.
In summary, our main contributions are as follows:
\begin{enumerate}
    \item We derive an upper bound for the convergence of non-convex loss functions using FL with arbitrary selection probabilities and no uniform bounded gradient assumption. We recover state-of-the-art convergence rates, including linear speedup with respect to the number of devices participating, with only a single-sided learning rate.
    \item We formulate a novel stochastic optimization problem that minimizes a weighted sum of the convergence bound and the amount of communication time spent on transmitting parameter updates, while satisfying transmit power constraints.
    \item Using the Lyapunov drift-plus-penalty framework, we derive an optimal online greedy solution that does not require knowledge of the channel statistics. We solve the problem by decoupling device selection from power allocation to obtain a combination of an analytical solution and a Riemannian manifold optimization problem that requires a solver due to a non-convex constraint.
    \item We provide experimental results that demonstrate a communication savings of up to $8.5\times$ compared to traditional uniform selection strategies and examine numerous scenarios involving varying levels of heterogeneity in both the data and the wireless channels.
\end{enumerate}

The rest of the paper is organized as follows.
First, we present some related work in Section \ref{sec:relatedWorks} before formally presenting our problem formulation in Section \ref{sec:probForm}.
Then, convergence analysis is provided in \ref{sec:convAnalysis} and the device scheduling policy is developed in Section \ref{sec:lyapunov}.
Finally, we present experimental results in Section~\ref{sec:experiments}.

%% file: RelatedWorks.tex
\section{Related Works}\label{sec:relatedWorks}
Since its introduction via the \emph{FedAvg} algorithm in \cite{mcmahan2017communication}, FL has garnered a lot of attention in both industry and academia.
Drawn by the promise of privacy, much work has been done to formally guarantee privacy \cite{mothukuri2021survey,geyer2017differentially,truex2019hybrid}, characterize convergence \cite{li2019convergence,karimireddy2020scaffold,mitra2021linear,li2018federated,wang2019adaptive}, and minimize communication overhead \cite{konevcny2016federated}.
For the latter, some strategies and analyses utilize model compression via sparsification \cite{han2020adaptive,sattler2019robust} and quantization \cite{alistarh2017qsgd,albasyoni2020optimal,reisizadeh2020fedpaq} to decrease the communication load while others have focused on optimizing device participation \cite{nishio2019client,cho2020client,fraboni2021clustered}, for example. 
Naively applying these approaches, e.g., biased compression or device selection, however, can lead to poor performance due to model skew which causes the model to drift away from the globally optimal solution.
The presence of heterogeneity in the system is the main cause of the skew which inhibits successful implementation of communication-efficient techniques like device selection at scale if not adequately addressed.
Heterogeneity can come in the form of differences in network and computational resources as well as from non-identical data distributions across the devices.
Overly relying on high-performing nodes during training is one way in which model skew can be unintentionally induced, resulting in a final model that works well for those nodes, but not for the lesser nodes whose data is ignored.

In fact, one of the first works specifically targeting device selection \cite{nishio2019client} suffers from this problem.
In this study, they report poor performance, especially for \noniid~datasets, since they only aggregate models from devices that respond the quickest during each training round.
This underlines the importance of accounting for and guaranteeing convergence in device scheduling algorithms.
Some other empirical studies without convergence results include \cite{ribero2020communication,goetz2019active}.
When dealing with heterogeneity, the challenge is how to quickly and sufficiently learn from all nodes to converge to a good global model without straining resources. 
Therefore, in this paper, we first focus on ensuring convergence of the global loss function before designing our device selection and power allocation algorithm.

Most theoretical works on the convergence of non-convex loss functions in FL assume uniform device participation via sampling with and without replacement \cite{li2019convergence,karimireddy2020scaffold,li2018federated,yang2021achieving} which does not allow for more intelligent device selection.
Meanwhile, other papers analyze more sophisticated selection schemes, but often with strict assumptions.
Among these is \cite{yang2019scheduling}, which analyzes the performance of three different scheduling policies but only contains convergence results for simple linear regression tasks. 
In \cite{cho2020client}, the authors analyze the convergence of a general, possibly biased, device selection strategy, but only for strongly convex loss functions and uniformly bounded gradients.
Importantly, their bound introduces a non-vanishing term to the convergence bound due to selection bias and thus their strategy is not guaranteed to converge to a stationary point of the loss function.
Both \cite{ren2020scheduling} and \cite{ruan2021towards} also consider convergence, but again only for strongly convex loss functions.
Uniquely, though, the work in \cite{ruan2021towards} considers cases where devices may compute incomplete updates, become inactive/unavailable, depart early, or arrive late, but does not provide specific selection strategies.
Finally, \cite{gu2021fast} considers arbitrary participation probabilities for each device, but in their case study, these probabilities are held constant throughout training and are not a design parameter.
In fact, the probabilities used in the experiments match the fraction of data samples each node has which follows the traditional \emph{FedAvg} approach.
Additionally, in \cite{gu2021fast}, all devices must participate in the first round for convergence.
A more recent theoretical treatment of arbitrary participation can be found in \cite{wang2022unified}.
We improve upon these results by considering non-convex loss functions and derive a bound with an easily managed non-vanishing term under the condition that all devices have an arbitrary non-zero probability of participating in each round.

The results most similar to ours can be found in \cite{fraboni2021clustered}.
In the paper, a ``clustered'' sampling approach is proposed where a multinomial distribution is used to sample devices.
This framework allows for the device selection probabilities to change with each draw in a given round. 
In order to remain unbiased, however, the probabilities must be chosen such that in each round, each node is selected on average proportional to their data ratio.
In our approach, by modifying the weights during the global aggregation step, we can maintain unbiased updates while allowing arbitrary participation probabilities at every round without restricting the values.
The increased flexibility enables us to choose probabilities depending on the state of the network to increase communication efficiency while still guaranteeing convergence.

As for FL over wireless networks, some works \cite{chen2020joint,yang2020energy} develop frameworks that jointly optimize convergence and communication.
Similarly to our approach, they derive a convergence bound and then minimize it by finding the optimal parameter values.
For example, in \cite{chen2020joint}, the FL loss is minimized while meeting the delay and energy consumption requirements via power allocation, user selection, and resource block allocation.
Both papers, however, make the unrealistic assumption that the channel remains constant throughout the training process which we do not assume here.
In \cite{wadu2020federated}, stochastic optimization is used to determine an optimal scheduling and resource block policy that simultaneously minimizes the FL loss function and CSI uncertainties.
The loss function considered, though, is simple linear regression and does not readily apply to neural network models.
Stochastic optimization is also considered for FL in \cite{huang2020efficiency} and \cite{zhou2020cefl}, but not to design an optimal device selection policy that guarantees convergence of non-convex loss function.
In summary, unlike previous work, our approach guarantees convergence of non-convex loss functions with arbitrary participation probabilities and uses the results to develop a device selection and transmit power allocation policy.
By leveraging Lyapunov optimization, we provide a partially analytical solution that increases communication efficiency of FL. 

%% file: Problem.tex
\section{Problem Formulation}\label{sec:probForm}
We now explain the FL optimization problem in more detail.
Consider a system with $N$ clients, where each client $n\in\{1,2,\dots,N\}$ has a possibly non-convex local objective $f_n(\x)$ with parameter $\x \in \mathbb{R}^d$.
We would like to solve the following finite-sum problem:
\begin{equation}\label{eqn:objFunction}
    \min_\x f(\x) := \frac{1}{N}\sum_{n=1}^N f_n(\x) .
\end{equation}
Note that for simplicity, we assume that each node has the same amount of data and thus have uniform weights, i.e., $\frac{1}{N}$, in \eqref{eqn:objFunction}.
When this is not the case, it can be easily accounted for by selecting appropriate weights for each local loss function in the sum and carrying it through the analysis.

To solve \eqref{eqn:objFunction}, we follow the typical \emph{FedAvg} \cite{mcmahan2017communication} FL paradigm, but modify it to allow for arbitrary device participation in each round.
In our algorithm, which can be found in Algorithm \ref{alg:fedavg}, participation is dictated by $q_n^t \in (0,1]$ which is the probability that device $n$ participates in round $t$.
To maintain unbiased updates, we weigh each device's gradient update inversely proportional to the participation probability of that device, as seen in Line 7 of Algorithm \ref{alg:fedavg}.
Intuitively, this ensures that devices that participate less frequently are weighed more heavily such that they have sufficient influence over the global model when they do participate and vice versa.
While knowledge of $q_n^t$ is required for aggregation, we treat it as a design parameter rather than a property of the environment and thus can conceptually develop an algorithm that chooses its value at each round.
We leave cases where $q_n^t$ is dictated by the environment to future work where it can be estimated.  
With this in place, we adjust its value for each device and aggregation round to optimize other aspects of the process such as system resource usage. 
Analyzing the convergence behavior of \emph{FedAvg} with arbitrary probabilities enables the ability to measure the impact of stochastic scheduling on training in order to better design device selection policies.

\begin{algorithm}[h]
 \caption{FedAvg with client sampling}
 \label{alg:fedavg}
\KwIn{$\gamma$, $\x_0$, $K$, $T$, $\{q_n^t\}$}
\KwOut{$\{\x_t\}$}

\For{$t \leftarrow 0, \ldots, T-1$}{
    Sample $\Identity_n^t \sim q_n^t, \forall n$;
    
    \For{$n \leftarrow 1,\ldots,N$ in parallel}{
        $\y^n_{t,0} \leftarrow \x_t$;

        \For{$i \leftarrow 0, \ldots, K-1$}{
            $\y^n_{t,i+1} \leftarrow \y^n_{t,i} - \gamma \g_n(\y^n_{t,i})$;
        }
        
    }
    
    $\x_{t+1} \leftarrow \x_t + \frac{1}{N}\sum_{n=1}^N \frac{\Identity_n^t}{q_n^t} \left(\y^n_{t,K}-\y^n_{t,0}\right) $; \, \tcp{global parameter update} \label{alg:weightAvg}

}
\end{algorithm}

We now explain Algorithm~\ref{alg:fedavg} in further detail.
First, $\Identity_n^t \in \{0,1\}$ denotes the random binary variable that represents whether client $n$ is selected in round $t$, where $q^t_n := \Pr\{\Identity_n^t = 1 \}$.
Next, after $\Identity_n^t$ is realized and the current global model is broadcasted, each device performs $K$ rounds of stochastic gradient descent (SGD) on its local dataset where $\g_n(\x)$ denotes the stochastic gradient of $f_n(\x)$ for device $n$ and $\gamma >0$ is the local learning rate.
This repeats for $T$ total aggregation rounds.
We use the variable $y^n_{t,i}$ to represent the intermediate local model updates between global aggregation rounds where $i$ indexes the local SGD iteration.
In the last step, only the selected devices' gradient updates are aggregated.
Note that even though this algorithm shows all $N$ devices performing SGD each round, it is logically equivalent to one in which only the selected clients via $\Identity_n^t$ receive the global model, compute gradient updates, and transmit back to the aggregator.

\section{Convergence Analysis}\label{sec:convAnalysis}
In this section, we provide an upper bound on the convergence of \eqref{eqn:objFunction} using Algorithm~\ref{alg:fedavg} for non-convex loss functions.
We assume that $\Identity_n^t$ and $\Identity_{n'}^t$ are independent through time, and that drawing from $\Identity_n^t$ is independent of SGD noise, i.e., $\Identity_n^t$ and $\g_n$ are independent.
We also make the following assumptions on the local loss functions.

\begin{assumption}[$L$-smoothness]\label{asmp:Lsmooth}
    \begin{align}
        \Vert \nabla f_n(\y_1) - \nabla f_n(\y_2) \Vert &\leq L \Vert \y_1 - \y_2 \Vert
    \end{align}
    for any $\y_1$, $\y_2$ and some $L>0$.
\end{assumption}
\begin{assumption}[Unbiased stochastic gradients]\label{asmp:unbiased}
    \begin{align}
        \Expectcond{\g_n(\y)}{\y}&=\nabla f_n(\y),
    \end{align}
    for any $\y$.
\end{assumption}

\begin{assumption}[Bounded stochastic gradient noise]\label{asmp:boundedGradNoise}
    \begin{align}
        \Expectbracket{\normsq{\g_n(\y) - \nabla f_n(\y)}} \leq \nu^2, \forall \y,n
    \end{align}
    for some $\nu > 0$.
\end{assumption}

\begin{assumption}[Bounded gradient divergence]\label{asmp:boundedGradDiv}
    \begin{align}
        \normsq{\nabla f_n(\y) - \nabla f(\y)} \leq \epsilon^2, \forall \y,n
    \end{align}
    for some $\epsilon > 0$.
\end{assumption}

While the first three assumptions are standard in non-convex optimization, the fourth is unique to FL and other distributed SGD techniques, as used in  \cite{li2018federated,yang2021achieving,lian2017can,reddi2020adaptive}. 
It bounds the differences between the local loss functions across clients due to having non-i.i.d.~datasets.
The i.i.d.~case is recovered when $\epsilon=0$.
Now, we state our convergence theorem in Theorem~\ref{thm:1}.
\begin{theorem}\label{thm:1}
    \input{theorem1_qmin}
\end{theorem}
\begin{proof}
    The full proof can be found in Appendix \ref{sec:appendix}.
    The proof utilizes the assumption that the client sampling random variable $\Identity_n^t$ is independent of the stochastic gradient noise in $\g_n$ and that $\Expectbracket{\Identity_n^t} = q_n^t$
    Thus, the weighted averaging of $\frac{1}{q_n^t}$ in Step \ref{alg:weightAvg} of Algorithm \ref{alg:fedavg} ensures that the gradient estimate remains unbiased.
    That is, $\Expectbracket{\frac{\Identity_n^t}{q_n^t} \g_n (\x)} = \nabla f_n(\x)$.

    Another key assumption is the existence of a constant $ 0<\qmin \leq q_n^t $ for all $n,t$.
    This is reasonable since, otherwise, $q_n^t \rightarrow 0$ for at least one $n$, which would cause the weight in the averaging step to go to infinity, trivially causing divergence.
    Additionally, $\qmin$ can be viewed as a design parameter, e.g., to control fairness.
\end{proof}

Our convergence bound consists of three terms.
The first vanishes as $T$ increases and, while the second and third are non-vanishing, they can be managed with appropriate selection of $\gamma$ (see Corollary \ref{cor:gen}).
The non-vanishing terms are separated into ones that are not amplified by partial participation, $\Phi_1$, and ones that are, $\Phi_2$.
The convergence bound increases when participation probabilities are small.
This follows with the intuition that sampling fewer participants in a round leads to a noisier estimate of the full gradient and thus more iterations will be required to converge to a desired loss.

Interestingly, the level of data heterogeneity, as measured by $\epsilon^2$, does not amplify the effect of partial participation.
This suggests that partial participation affects convergence solely through increasing the variance of the gradient estimate and not by a biasing of the model towards the minima of the more frequent participants.
This decoupling is due to the inverse-proportional weighting of participation probabilities in the aggregation step of Algorithm \ref{alg:fedavg} which ensures unbiased gradient estimates.
It essentially compensates for infrequent participants by taking larger steps towards their local minima in the low-probability event that they are called upon.
Similarly, the upper bound on step size requires that when the minimum participation probability, $\qmin$, is small, the step size must also be small.
This makes sense intuitively since larger steps would also heavily skew the model towards the more frequent participants' local stationary points.

In the bound, the partial participation term can be trivially minimized by setting $q_n^t=1$ for all $n$ and $t$, i.e., full participation, but it is impractical to assume that every device can or will participate in every round.
This necessitates a selection strategy that minimizes the time average $\frac{1}{T} \sum_{t=0}^{T-1} \frac{1}{N} \sum_{n=1}^N \frac{1}{q_n^t}$ under some constraints.
Before introducing our approach to this problem, we present two corollaries to further examine convergence behavior and show that linear speedup is achieved.

\begin{corollary}\label{cor:gen}
    If we choose $\gamma = \min\{\frac{\qmin}{8LK},\frac{\sqrt{N\qmin}}{\sqrt{TK}L}\}$, then we have
    \begin{align}
        &\frac{1}{T}\sum_{t=0}^{T-1}\Expectbracket{\normsq{\nabla f(\x_t)}} \nonumber\\
        & \quad\leq \mathcal{O}\left(\frac{L}{\qmin T} + \frac{L}{\sqrt{TNK\qmin}} + \frac{N\qmin \epsilon^2}{TK} + \frac{\sqrt{\qmin} \nu^2 Q}{\sqrt{TNK}}\right),
    \end{align}
    where $Q=\frac{1}{T}\sum_{t=0}^{T-1}\frac{1}{N} \sum_{n=1}^N \frac{1}{q_n^t} \leq \frac{1}{\qmin} $.
\end{corollary}
\begin{proof}
    The order terms can be obtained by plugging in the step size in \eqref{eqn:convBound} and noting that $\max\{x,y\}\leq x+y$.
\end{proof}

    \begin{corollary}\label{cor:full}
        In the full participation scenario, i.e.,~$q_n^t=1 \, \forall n,t$, the convergence rate of Algorithm \ref{alg:fedavg} is $\frac{1}{T}\sum_{t=0}^{T-1}\Expectbracket{\normsq{\nabla f(\x_t)}} \leq \mathcal{O}\left(\frac{1}{\sqrt{TNK}}\right)$. 
   
         Furthermore, for uniform sampling \emph{with} and \emph{without} replacement, the convergence rate of Algorithm \ref{alg:fedavg} is $\frac{1}{T}\sum_{t=0}^{T-1}\Expectbracket{\normsq{\nabla f(\x_t)}} \leq\mathcal{O}\left(\frac{1}{\sqrt{TmK}}\right)$. 
    \end{corollary}
    \begin{proof}
        For the full participation and uniform sampling without replacement scenarios, the corollary can be obtained by noting that $\qmin = q_n^t=1$ and $\qmin = q_n^t=m/N$ for all $n,t$, for each scenario, respectively, and then applying Corollary \ref{cor:gen} with $Q=1/\qmin$.
        
        For the uniform sampling with replacement scenario, we note that $ \qmin = q_n^t=1-(1-1/N)^m$. 
        To show the result, we first prove that $ 1/\qmin = \mathcal{O}(N/m)$ for sampling with replacement and then the result follows from the sampling without replacement strategy.
        
        First, consider the case where $m\geq N$,\footnote{In sampling \emph{with} replacement, since devices can be re-selected each draw, you can continue to sample beyond the number of devices and thus it is possible for $m\geq N$. For sampling \emph{without} replacement, $m=N$ would result in full participation and no more draws can be made.} then we have
        \[
            1-\left(1-\frac{1}{N}\right)^m \geq 1-e^{-m/N} \geq 1-e^{-1}.
        \]
        Next, consider the case where $m\leq N$ and let $r=m/N$ such that $r\leq 1$.
        Observe, then, that the derivative with respect to $r$ of $(1-(1-r/m)^m)/r$ is always negative and thus is minimized when $r=1$.
        Then, we have
        \begin{align*}
            \left(1-\left(1-\frac{1}{N}\right)^m\right)\frac{N}{m} &= (1-(1-r/m)^m)/r \\
            & \geq 1-(1-1/m)^m \\
            & \geq 1-e^{-1},
        \end{align*}
        which gives
        \[
            1-\left(1-\frac{1}{N}\right)^m \geq \frac{m}{N} (1-e^{-1}) .
        \]
        Combining the two cases, we have a final bound of
        \[
            1-\left(1-\frac{1}{N}\right)^m \geq \frac{\min\{m,N\}}{N} (1-e^{-1}),
        \]
        and thus,
        \[
            1/\qmin = \frac{1}{1-\left(1-\frac{1}{N}\right)^m} = \mathcal{O}\left( \frac{N}{m} \right).
        \]
        The result then follows simply via Corollary \ref{cor:gen}.
        Finally, we note that if $m>N$, then the $m$ in the order term is replaced with $N$.
        
    \end{proof}


    Corollary \ref{cor:full} shows that we recover the state-of-the-art convergence rate in \cite{li2019convergence,karimireddy2020scaffold} that indicates a linear speedup in terms of the number of clients for both full \emph{and} partial participation.
    This is achieved \emph{without the use of a two-sided learning rate}\footnote{Although, our weighted averaging step in \eqref{alg:weightAvg} may be viewed as a per-device, per-round learning rate that becomes a global learning rate in the uniform sampling scenario.} which prior methods require.
    Linear speedup is desired since the number of iterations required to reach a given loss can be reduced by proportionally increasing the number of clients in a given round.
    This behavior ensures that our algorithm fully leverages the parallelism of federated learning, even with \noniid~data.
    For general participation probabilities, the speedup is in terms of $N\qmin$.
    Interestingly, with the choice of step size in Corollary \ref{cor:gen}, if $\qmin \leq 1/N$, then speedup is lost.

Now equipped with an understanding of how device selection probability $q_n^t$ affects FL convergence, we will design an optimization problem that adaptively determines both $q_n^t$ and transmit power allocations in order to minimize communication overhead while guaranteeing convergence.

%% file: theorem1_qmin.tex
Let Assumptions \ref{asmp:Lsmooth}--\ref{asmp:boundedGradDiv} hold with $\gamma$, $T$, $K$, $N$, and $q_n^t$ defined as above.
Then, if $\gamma \leq \frac{\qmin}{8LK}$, where we assume the existence of a minimum participation probability $\qmin$ such that such that $ \qmin \leq q_n^t$ for all $n,t$, Algorithm~\ref{alg:fedavg} satisfies
\begin{align}
    \frac{1}{T}\sum_{t=0}^{T-1}\Expectbracket{\normsq{\nabla f(\x_t)}} \leq & \frac{2\left(\Expectbracket{f(\x_0)} - \Expectbracket{f(\x_{T})}\right)}{c\gamma KT} \nonumber\\
    &\quad + \Phi_1 +  \frac{\Phi_2}{TN} \sum_{t=0}^{T-1} \frac{1}{N} \sum_{n=1}^N \frac{1}{q_n^t} \, , \label{eqn:convBound}
\end{align}
where $\Phi_1 = \frac{1}{c}5 \gamma^2 K L^2\left(\nu^2+6K\epsilon^2 \right)$, $\Phi_2 = \frac{2L\gamma\nu^2}{c}$, and $c$ is a constant.

%% file: Algorithm.tex
\section{Communication-Efficient Scheduling Policy}\label{sec:lyapunov}
In this section, we propose a novel device selection and transmit power allocation policy that minimizes communication overhead while guaranteeing convergence. 
The policy is based on solving a stochastic optimization problem that minimizes a time average function consisting of the convergence bound in \eqref{eqn:convBound} and the average time spent communicating model updates, while satisfying time average transmission power constraints for each device.
This optimization problem is formally stated in Section \ref{sec:prob_form}.
The output of the optimization problem is the selection probabilities $q_n^t$ used in Algorithm \ref{alg:fedavg} and the transmit powers $P_n^t$ used in each round.
The communication of multiple devices' model updates over many rounds may cause a huge bottleneck in FL, especially in heterogeneous environments, which is why we focus on minimizing the time spent during this phase.
Unlike previous work, we speed up training in terms of wall-clock time, minimizing the burden on the network, while still theoretically guaranteeing convergence.
The natural form of the bound allows us to formulate the problem in terms of time averages.
This lends itself perfectly to the application of the Lyapunov drift-plus-penalty framework which specifically deals with this class of problems.
A major benefit of the approach is that it does not require knowledge of the exact dynamics or statistics of the channel; only the instantaneous channel state information (CSI) is needed.

\subsection{Communication Model}
In our model, we consider a simple wireless star network where all devices have direct uplink and downlink channels to the global aggregator (see Figure \ref{fig:blockDiagram}).
Each device has its own channel $h_n^t$ to the aggregator, but must take turns in transmitting their parameters with transmit power $P_n^t$ via time-division multiple access (TDMA).
For simplicity, we only consider the uplink communication time as the downlink is a broadcast of common information by a highly capable base station which takes much less time.
By assuming that we can transmit at capacity, the communication time for device $n$ in a given round $t$ is
\begin{align}\label{eqn:comm_time}
     \tau_n^\text{comm}(t) = \frac{\ell \, \Identity_n^t}{B\log_2\left( 1\!+\!|h_n^t|^2 \frac{P_n^t}{N_0} \right)}\,,
\end{align}
where $\ell$ is the size of the model in bits, $B$ is the bandwidth, and $\log_2(\cdot)$ denotes the base 2 logarithm.
We assume that the aggregator has current CSI in the form of channel gain $|h_n^t|^2$ and noise power $N_0$ at each round $t$ in order to calculate this, but not the underlying distributions.

Next, as is typical in FL \cite{mcmahan2017communication,karimireddy2020scaffold,yang2021achieving}, we assume that the aggregator is using sampling \emph{with} replacement with $m$ draws.
With this policy, a minimum of one client and a maximum of $m$ clients are selected in each round.
More specifically, the aggregator draws from a multinomial distribution where each device has probability $\omega_n^t$ of being chosen in each of the $m$ draws such that $\sum_{n=1}^N \omega_n^t = 1$ for every round $t$.
Our goal is to determine $\omega_n^t$ such that communication time is minimized and convergence is guaranteed.
However, since our convergence bound is in terms of arbitrary marginal probabilities $q_n^t$, we must relate $\omega_n^t$ to $q_n^t$ in order to measure the policy's impact on convergence.
For $m$ draws, the relationship is
\begin{align}
    q_n^t &= 1-(1-\omega_n^t)^m,
\end{align}
where we now require
\begin{align}
    \sum_{n=1}^N \omega_n^t =\sum_{n=1}^N  1-(1-q_n^t)^{1/m} = 1
\end{align}
to be satisfied in order to obtain a valid multinomial distribution.
In practice, sampling \emph{without} replacement can be used, but expressing $q_n^t$ in terms of arbitrary $\omega_n^t$ and $m$ is not tractable and thus difficult to optimize directly.

\subsection{Problem Formulation}\label{sec:prob_form}
Similar to some prior works, such as \cite{perazzone2022communication,wang2019adaptive,wang2022federated,shi2020joint,cui2022optimal}, we minimize the convergence bound \eqref{eqn:convBound} as a proxy for the actual convergence loss since the direct effect that parameters have on loss is generally unknown.
Simultaneously, we also wish to minimize the time spent communicating the parameters, as in \eqref{eqn:comm_time}, to decrease wall-clock training time.
The former is minimized when all devices participate in every round while the latter is minimized when fewer devices participate.
Thus, to balance the competing goals, we formulate our objective function as the weighted sum between the two:
\begin{align}\label{eqn:objFunc}
    y_0(t,q_n^t,P_n^t) :=  \sum_{n=1}^{N} \left(\frac{1}{N q_n^t} + \lambda\!\cdot\! \tau_n^\text{comm}(t) \right), 
\end{align}
where $\lambda>0$ is a user-defined, tunable parameter that controls the trade-off.
The parameter also subsumes the coefficients in the last term in \eqref{eqn:convBound} such that $L$ and $\nu$, for example, can be ignored in the following optimization problem.

Finally, we wish to limit the power expenditure of communication over time to a given time-average threshold, $\bar{P}_n$, and instantaneous power $P_\text{max}$.
Thus, we formulate our optimization problem as
\begin{align}\label{eqn:minProb}
    \min_{\{q_n^t\},\{P_n^t\}} \quad & \lim_{T\rightarrow\infty} \frac{1}{T} \sum_{t=0}^{T-1} \Expectbracket{y_0(t)}\\
    \text{s.t.} \quad & \lim_{T\rightarrow\infty} \frac{1}{T} \sum_{t=0}^{T-1} P_n^t q_n^t \leq \bar{P}_n, \ \forall n=1,\ldots,N \nonumber\\
    & 0 \leq P_n^t \leq P_\text{max}, \ n=1,\ldots,N \nonumber\\
    & \sum_{n=1}^N  1-(1-q_n^t)^{1/m} = 1 & \nonumber\\
    & q_n^t \in (0,1], \nonumber
\end{align}
where
\begin{align}\label{eqn:objFunc}
    \Expectbracket{y_0(t)} &= \Expectbracket{y_0(t,q_n^t,P_n^t)} \nonumber \\
    & = \sum_{n=1}^{N} \left(\frac{1}{N q_n^t} + \lambda\!\cdot\! \frac{\ell \, q_n^t}{B\log_2\left( 1\!+\!|h_n^t|^2 \frac{P_n^t}{N_0} \right)}\right).
\end{align}

The first two constraints limit the time average and peak power, respectively, where $ P_n^t q_n^t = \Expectbracket{P_n^t \Identity_n^t} $ is the expected power usage in round $t$ by device $n$.
The third constraint ensures that a proper multinomial distribution is found for sampling with replacement.
In its current state, \eqref{eqn:minProb} is difficult to solve as it is not possible to know the behavior of the channel $h_n^t$ for all $t$ ahead of time.

\subsection{Lyapunov Formulation}
The novelty of our convergence bound and formulation comes from the fact that both the effect of $q_n^t$ on convergence and the communication time are in the form of a time average.
This allows us to leverage the Lyapunov stochastic optimization framework \cite{neely2010stochastic} to reformulate \eqref{eqn:minProb} into a form that we can solve greedily at each round with optimality guarantees.
By converting our transmission power constraints into a set of \emph{virtual} queues, we can apply the Lyapunov theory to analyze our problem and derive an online solution.
The practical implications of stabilizing virtual queues will be explored at the end of this section and its effect will be further illustrated in the experiments of Section \ref{sec:experiments}.

To put our optimization problem into the Lyapunov drift-plus-penalty framework and using standard notation, we turn the first time-average constraint in \eqref{eqn:minProb} into a virtual queue $Z_n^t$ for each client $n$ such that
\begin{align}\label{eqn:virtualQUpdate}
    Z_n^{t+1} = \max [Z_n^t+y_n(t),0] \,,
\end{align}
where \vspace{-0.1in}
\begin{align}
    y_n(t) = P_n^t q_n^t - \bar{P}_n \,.
\end{align}
Since we have no actual queues, the Lyapunov function is 
\begin{align}
    \mathcal{L}(\mathbf{\Theta}^t) := \frac{1}{2} \sum_{n=1}^N (Z_n^t)^2 \,,
\end{align}
where $\mathbf{\Theta}^t$ represents the current queue states, which in this case, is just $\{Z_n^t: \forall n \}$.
Next, we define the Lyapunov drift:
\begin{align}
    \Delta^{t+1} = \mathcal{L}^{t+1}- \mathcal{L}^{t},
\end{align}
where we drop $\mathbf{\Theta}^{t}$ for simplicity.
Finally, we have the Lyapunov drift-plus-penalty function that we aim to minimize:
\begin{align}\label{eq:driftpluspenalty}
    \Delta^{t} + V \Expectbracket{y_0(t)|\mathbf{\Theta}^{t}},
\end{align}
where $V>0$ is another arbitrarily chosen weight that controls the fundamental trade-off between queue stability and optimality of the objective functions (the effect of which is also explored in Section \ref{sec:experiments}).

Now, by utilizing Lemma 4.6 from \cite{neely2010stochastic} and assuming that the random event, i.e., channel gain $|h_n^t|^2$, is \iid\footnote{We make this assumption out of simplicity. Additional analyses can show that the algorithm converges even when this doesn't hold, including for non-ergodic processes \cite{neely2010stochastic}.}~with respect to $t$, we can upper bound \eqref{eq:driftpluspenalty}:
\begin{align}\label{eqn:driftPlusPenaltyBound}
    \Delta^{t} + V \Expectbracket{y_0(t)|\mathbf{\Theta}^{t}} &\leq C +  V \Expectbracket{y_0(t)|\mathbf{\Theta}^{t}}\nonumber\\
        & \quad+ \sum_{n=1}^N Z_n^t\Expectbracket{y_n(t)|\mathbf{\Theta}^{t}}
\end{align}
where $C>0$ is a constant. 
Next, according to the Min Drift-Plus-Penalty Algorithm, we opportunistically minimize the expectation in the right hand side of \eqref{eqn:driftPlusPenaltyBound} at each time step $t$:
\begin{align}
    \min_{\{\omega_n^t\},\{P_n^t\}}\,\,\,  & f(q_n^t,P_n^t) := V y_0(t) + \sum_{n=1}^N Z_n^ty_n(t) \label{eqn:minimization}\\
    \text{s.t.}\quad  & 0 \leq P_n^t \leq P_\text{max}, \quad \forall n=1,\ldots,N \nonumber\\
    & \sum_{n=1}^N \omega_n^t  = 1 \nonumber\\
    & q_n^t = 1-(1-\omega_n^t)^{m} \nonumber \\
    & \omega_n^t \in (0,1] \, . \nonumber
\end{align}
We now have the problem in a form that is decoupled in time allowing for a greedy optimization approach.







        
        
        
        

    

\subsection{Solving the Problem \eqref{eqn:minimization}}
First, we utilize the fact that \eqref{eqn:minimization} is convex with respect to $P_n^{t}$ and that the optimal value for $P_n^{t}$ is independent of $q_n^t$ to produce the following analytical expression in Theorem \ref{thm:Lyapunov}.
\begin{theorem}\label{thm:Lyapunov}
    The $P_n^{t}$ that optimizes \eqref{eqn:minimization} is independent of $q_n^t$ and is given by either the endpoint, i.e., $P_n^{t,\textnormal{opt}}=P_\text{max}$, or by
    \begin{align} \label{eqn:powerOpt}
        P_n^{t,\textnormal{opt}} = \frac{N_0}{|h_n^t|^2} \left(\frac{A}{4} W_0\left(\sqrt{\frac{A}{4}}\right)^{-2}-1\right)
    \end{align}
    where $A=\frac{V\lambda\ell |h_n^t|^2 \left(\log(2)\right)^2}{N_0 B Z_n^t}$ and $W_0(\cdot)$ is the principal branch of the Lambert $W$ function.
\end{theorem}
\begin{proof}
    The proof can be found in Appendix \ref{sec:appendix2}.
\end{proof}

Optimizing over $q_n^t$ is trickier because of the additional constraint.
We first note that the problem \eqref{eqn:minimization} can be simplified to the following form
\begin{align}
    \min_{\{\omega_n^t\}}\,\,\,  & \sum_{n=1}^{N} \left(A_n^t\left(q_n^t\right)^{-1} + B_n^t q_n^t + C_n^t\right) \label{eqn:minSimple}\\
    \text{s.t.}\quad  & \sum_{n=1}^N \omega_n^t  = 1 \nonumber\\
    & q_n^t = 1-(1-\omega_n^t)^{m} \nonumber \\
    & \omega_n^t \in (0,1] \, , \nonumber
\end{align}
where $A_n^t = \frac{V}{N}$, $ B_n^t = \frac{V\lambda\ell}{B\log_2\left( 1+|h_n^t|^2 \frac{P_n^{t,\textnormal{opt}}}{N_0}\right)}+Z_n^tP_n^{t,\textnormal{opt}}$, $C_n^t = -Z_n^t \bar{P}_n$, and $P_n^{t,\textnormal{opt}}$ is the optimal $P_n^t$ determined by \eqref{eqn:powerOpt}.
While the objective function is convex, we have introduced a non-convex equality constraint which makes this problem significantly more difficult to solve.
Additionally, the constraint creates a coupling of the individual device selection probabilities and cannot be solved independently.

A natural approach to solving this optimization problem is to use the Lagrangian multiplier method.
Unfortunately, however, the gradient of the Lagrangian of \eqref{eqn:minSimple} does not have a simple closed form solution, so we must resort to numerical solvers.
The feasible region of this problem is the probability simplex which is a Riemannian manifold.
To solve this problem, we use Manopt \cite{manopt} which is a toolbox for optimization on manifolds.

\subsection{Optimality Guarantees}
Theorem 4.8 in \cite{neely2010stochastic}, Theorem \ref{thm:Lyapunov}, and assuming we obtain the optimal value for $q_n^t$ from the solver guarantee that this algorithm satisfies
\begin{align}\label{eqn:optimGap}
    \limsup_{T\rightarrow\infty} \frac{1}{T} \sum_{t=0}^{T-1} \Expectbracket{y_0(t)} \leq y_o^{\textnormal{opt}} + \frac{C+D}{V},
\end{align}
where $y_o^{\textnormal{opt}}$ is the minimum of $y_o$, $D\geq0$ is a constant whose existence is guaranteed in Theorem 4.8, and $C\geq0$ is a constant that bounds the expected difference between the optimal value of \eqref{eqn:minimization} and the value given by an approximately optimal decision.
Since the subproblem in \eqref{eqn:minSimple} is non-convex, we may choose a local optimum, so $C\neq0$ but is finite.
The theorem also guarantees that the time-average transmit power constraint is satisfied as $t\rightarrow\infty$.
The user-defined parameter $V$ that traditionally controls the trade-off between the average queue backlog and the gap from optimality now controls the speed of convergence in addition to the optimality gap in~\eqref{eqn:optimGap}.

In \eqref{eqn:powerOpt}, we can see that when there is a large virtual queue $Z_n^t$, the transmission power is decreased in order to satisfy the constraint.
In this way, the virtual queue represents how far from the time average constraint we are.
As $V$ is increased, the effect that the current virtual queue has on selection becomes less important and it takes longer to satisfy the average power constraint.
This is also explored experimentally in Section \ref{sec:effectV}.
A large $\lambda$ favors the minimization of communication time rather than the convergence bound which naturally leads to more unbalanced $q_n^t$, whereas small $\lambda$ will results in more uniform $q_n^t$.

%% file: Experiments.tex
\section{Experiments}\label{sec:experiments}
In our experiments, we trained a convolutional neural network (CNN) on the CIFAR-10 benchmark \cite{krizhevsky2009learning} with data spread across $N=100$ clients.
We examine the convergence behavior over wall-clock time for numerous scenarios with varying levels of data and channel heterogeneity, number of clients draws $m$, and values for hyperparameter $\lambda$.
We then compare our device scheduling algorithm with the uniform selection baseline to demonstrate the speedup in convergence ours provides.
Since the cost of an iteration/aggregation round is variable, analyzing convergence in terms of wall-clock time instead of iterations gives more insights into how the algorithms will perform in practice.
The standard assumption that it is better to converge in fewer iterations does not necessarily hold if the cost of each iteration is large.
The wall-clock time consists of both communication time and computation time where communication time is calculated as in \eqref{eqn:objFunc} and computation time is chosen to be a constant.
It is assumed that computation is performed in parallel and each device completes all $K$ local iterations at the same deterministic time as opposed to communication which occurs sequentially via TDMA.

\subsection{Setup}
The CIFAR-10 dataset consists of 60,000 color images of 10 classes where 10,000 images are reserved for testing.
We assign each device 500 samples from the dataset.
In order to simulate a \noniid~data distribution among devices, we use the approach presented in \cite{hsu2019measuring} where the distribution of the 10 classes among each client is determined by a categorical distribution with a Dirichlet prior.
This allows for the level of heterogeneity to be quantified by a single parameter $\alpha$.
More specifically, for each experiment and client, we draw the probabilities for the categorical distribution via $\mathbf{c}\sim \text{Dir}(\alpha \mathbf{e}) $ where $\mathbf{e}$ is a uniform probability vector.
We can then vary $\alpha$ to produce different levels of heterogeneity.
For example, at the extremes, we have $\alpha \rightarrow 0$ which results in very \noniid~data where $\mathbf{c}$ becomes a one-hot vector such that each client is assigned samples from only one class.
Meanwhile, on the other hand, when we have $\alpha \rightarrow \infty$, we encounter the \iid~case where each client gets an even distribution of all classes.
Then, after obtaining $\mathbf{c}$ for each client, we draw from the resulting categorical distribution $500$ times to determine the class of each sample.
When the class is determined, one sample from that class is uniformly drawn from the dataset.
We note that this does allow for the possibility of certain samples to be assigned more than once or not at all, but this approach provides the most randomness and flexibility.

For training, we use the same CNN model as in \cite{perazzone2022communication} which contains $d=555,178$ parameters.
Assuming $32$-bit floating point numbers, we accordingly set $\ell=32d$.
We also set the minibatch size to $32$, $\gamma = 0.01$, $I=10$, and $B=22\times 10^{6}$ to simulate WiFi bandwidth.
The power constraints are set to $\bar{P}_n=1$ and $\Pmax=35$ dB and noise power is normalized to $N_0=1$.
We set $V=100$ in our selection algorithm and justify this choice later in Section \ref{sec:effectV}.
For the channel model, we assume each device experiences Rayleigh fading such that $|h_n^t|$ is distributed as a Rayleigh random variable.
In the first set of experiments, we examine the heterogeneous channel case such that each device experiences a different level of fading.
In particular, we assign the Rayleigh parameters in a linear increasing fashion from $\sigma=0.1$ to $\sigma=10$ for the $100$ devices.
For the next set of experiments, we examine the homogeneous channel case such that each device has $\sigma=1$.
Note that our algorithm only requires knowledge of the instantaneous $|h_n^t|$ and not the underlying distribution, requiring only that it is \iid~across time slots.

To fairly compare to the uniform selection baseline, we set $q_n^t$ to be the appropriate values for the aggregation weights in Algorithm \ref{alg:fedavg} and we set $P_n^t$ to maximally satisfy the transmit power constraint in \eqref{eqn:minProb}.
So, for uniform sampling \emph{with} replacement, we set $q_n^t=\frac{1}{1-(1-1/N)^{m}}$ and $P_n^t=\min\{\Pmax, \bar{P}_n \cdot q_n^t \}$.
To avoid big outliers that likely would not be chosen by either selection policy in practice, we lower bound the possible values for $|h_n^t|^2 \geq .001$.
This is very generous to the uniform case as it is agnostic to channel conditions and would greatly suffer from choosing a device in deep fade.
Finally, since the total elapsed time varies across runs at each communication round, we perform linear interpolation for each run to generate points at uniform time intervals.
This allows for us to more easily average the performance of the $3$ runs for each scenario.
We then perform a rolling average to smooth the curves to better see the convergence trends.

\subsection{Exploring the impact of computation time}
First, we give an illustrative example of how the optimal number of sampling with replacement draws $m$ in each round is impacted by increasing computation time.
In Fig. \ref{fig:time_to_target_acc}, we plot the time to reach a target accuracy of $0.775$ versus computation time per round for a basic scenario with $m= 1$, $5$, and $10$ draws using our algorithm.
In this scenario, there are three clear regions in which a different number of draws reaches the target accuracy quickest.
It can be seen that as computation time increases, the optimal number of device draws also increases. 
In the first region, \emph{scheduling only one device per round is optimal} despite the fact that it takes more iterations to converge than the other cases (depicted later in Fig. \ref{fig:Het_CIFAR10_IID_iter}). 
However, since the communication phase in each iteration takes less time, due to the aggregator only needing to wait for $1$ device in a sequential communication process, the lower quality gradient estimates are overcome by the quick communication.
On the other hand, when computation time is comparatively larger than communication, drawing more devices per round is optimal. 
This is because as computation time increases, it begins to dominate the total time, so communicating additional devices' updates in each round has a smaller comparative impact on total time.
Thus, \emph{it is more advantageous to schedule more devices in each round} in order to leverage the parallelism of computation and obtain more informative updates.

We expect that as computation time continues to increase, the optimal number of draws will continue to increase up to a certain point.
It is hard to predict this point as it is related to optimal minibatch sizes. 
This is a similarly unsolved problem where the generalizability of smaller minibatches generally outweighs the benefits of more accurate gradient estimates obtained from larger batch sizes \cite{keskar2016large}.
Additionally, at a certain point Amdahl's Law must be taken into consideration for very large computation times.
Even though our algorithm still technically provides an advantage over uniform, if communication only accounts for a small percentage of the total time, then optimizing communication alone cannot speed up performance significantly.

\begin{figure}
    \centering
    \includegraphics[width=0.6\linewidth]{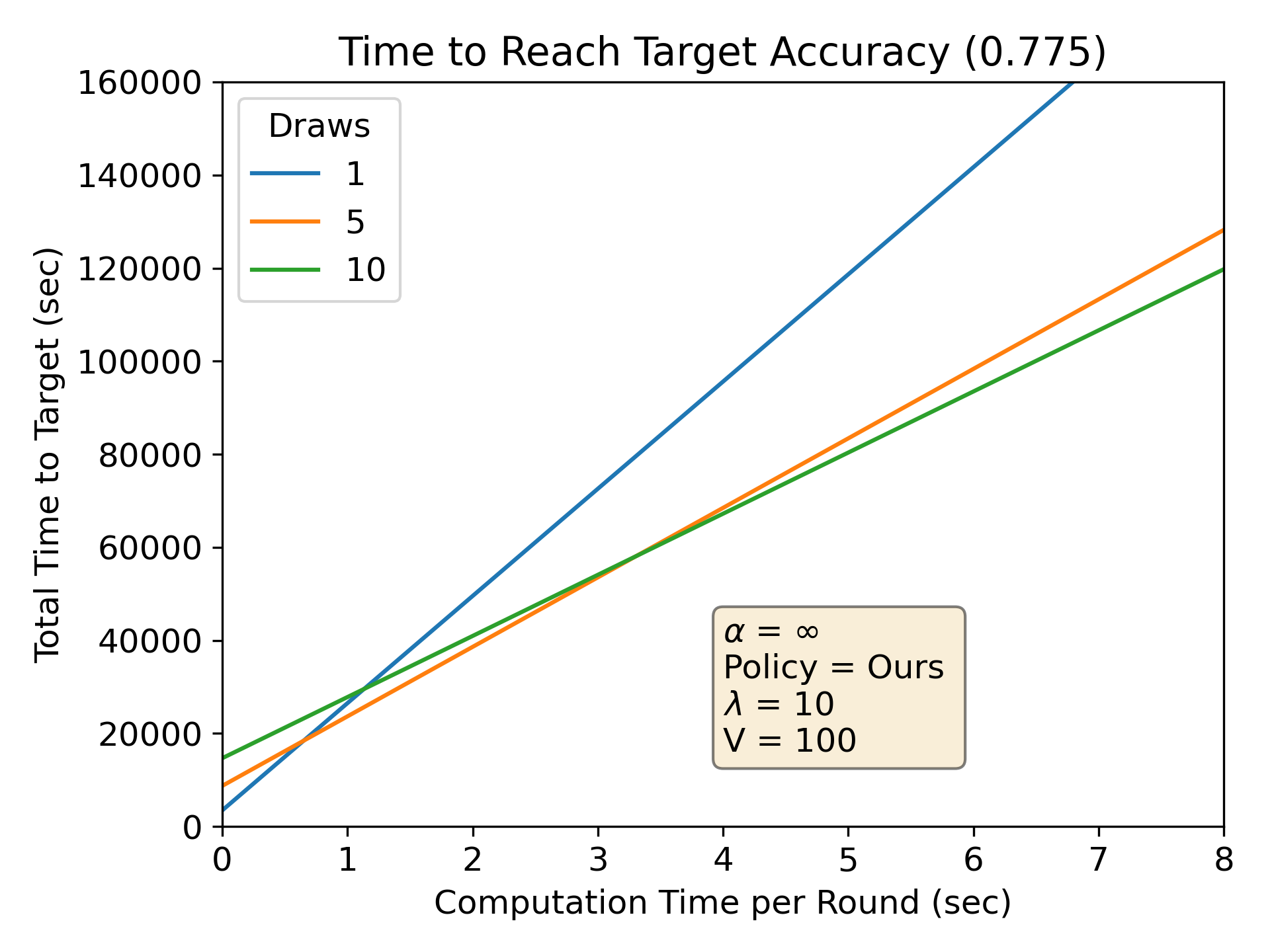}
    \caption{The optimal number of devices chosen each round depends on computation time.}
    \label{fig:time_to_target_acc}
\end{figure}

\subsection{Best Hyperparameters}
Next, in Table \ref{tab:opt_config}, we identify the best hyperparameters for both our algorithm and the uniform baseline for $3$ different levels of data heterogeneity over computation time ranges.
We see similar trends for all three levels of data heterogeneity and both selection policies where, again, \emph{higher computation time results in a greater number of draws}.
It is also interesting to note the optimal $\lambda$ values.
For large $\lambda$, minimizing communication time is favored over minimizing the convergence bound.
Here, we see that as computation time increase comparative to communication time, lower $\lambda$ values become optimal.
This is because as the percentage of total time due to computation time increases, communication time becomes less important.
Thus, it is better to choose devices in a more uniform manner in order to avoid client drift toward devices with better channels.
We will use these optimal hyperparameter ranges in the following results, but we note that we only tested a discrete subset of hyperparameters so they are not necessarily truly optimal.

\begin{table}[h]
    \centering
    \caption{Best hyperparameters from experiments for different computation time ranges and levels of data heterogeneity.}
    \begin{tabular}{ |c||c|c|c| }
         \hline
         \multicolumn{4}{|c|}{Heterogeneous Channel Gain} \\
         \hline
         \multicolumn{1}{|c||}{}&\multicolumn{3}{c|}{Ours}\\
         \hline
         Case& Comp Times (s) & Draws ($m$) & $\lambda$ \\
         \hline
         $\alpha=\infty$, & $0 - 0.83$   &  $1$  & $100$ \\
         target $=0.775$  &   $0.83 - 2.3$  & $10$   & $100$  \\
         & $2.3 - 7.5$ & $5$&  $10$  \\
         & $7.5 - 16.3$ & $10$ & $10$    \\
         &  $16.3 - \infty$  & $10$ & $1$   \\
         \hline
         $\alpha=1$, & $0 - 0.61$  & $1$  & $100$ \\
         target $=0.73$ & $0.61 - 3.1$ & $10$ & $100$\\
         & $3.1 - 15.6$ & $5$ & $1$\\
         & $15.6 - \infty$ & $10$ & $1$   \\
         \hline
         $\alpha=0$, & $0 - 0.26$ & $1$ & $100$\\
         target $=0.65$ & $0.26 - 1.8$ & $10$ & $100$\\
         & $1.8 - \infty$ & $10$ & $10$ \\
         \hline
         \hline
         \multicolumn{1}{|c||}{}&\multicolumn{3}{c|}{Uniform}\\
         \hline
         Case& Comp Times (s) & Draws ($m$) &  \\
         \hline
         $\alpha=\infty$, & $0 - 2$ & $1$ & \\
         target $=0.775$ & $2 - \infty$  &  $5$ & \\
         \hline
         $\alpha=1$,  & $0 - 0.84$  &  $1$ & \\
         target $=0.73$ & $0.84 - 46$  &  $5$ & \\
         & $46 - \infty$  & $10$ & \\
         \hline
         $\alpha=0$, & $0 - 0.74$ & $1$  & \\
         target $=0.65$ & $0.74 - 6.$1 & $5$  & \\
          & $6.1 - \infty$ & $10$ & \\
         \hline
    \end{tabular}
    \label{tab:opt_config}
\end{table}

\subsection{Heterogeneous Channels}
In this section, we investigate the convergence behavior of our algorithm over wall-clock time and compare it to the uniform baseline.
We begin with the heterogeneous channels case where some devices have very good channels on average while others do not.
In Fig. \ref{fig:CIFAR10}, we show how top-$1$ accuracy and training loss are affected by different sampling with replacement draws $m$.
In Figs. \ref{fig:Het_CIFAR10_IID} and \ref{fig:Het_CIFAR10_IID_loss}, we show the results for instantaneous computation such that the total time only consists of communication time while in Figs. \ref{fig:Het_CIFAR10_IID2} and \ref{fig:Het_CIFAR10_IID_loss2}, the computation time is set to $2$ seconds.
The plots clearly show that \emph{our algorithm outperforms uniform}, especially when the number of draws is large.
In fact, for the $m=10$ regime, we see a very large speedup of $8.5\times$ to reach a target accuracy of $0.775$ where our policy takes $5,339$ seconds on average and uniform takes $45,748$ seconds.
The speedup is lesser for the $m=1$ case which is about $1.3\times$ for the same target.
In these first two plots, where only communication time is taken in consideration, scheduling \emph{just one device per round} is quickest.
As explained in the previous section, this is a scenario where it is \emph{better to quickly aggregate updates from fewer devices} since the shorter communication time per round outweighs the higher variance of each update.

In the second set of plots, we show the same scenario but with parallel computation time of two seconds per round.
This causes the optimal number of draws to change to $m=5$ with a speedup of $1.31\times$ over uniform.
The speedup in the $m=10$ case remains large, however.
Since optimizing communication time only affects the fraction of the total time due to communication, the speedup is less significant when computation times are large.

\begin{figure}
    \centering
    \begin{subfigure}[b]{0.49\linewidth}
        \centering
        \includegraphics[width=1\linewidth]{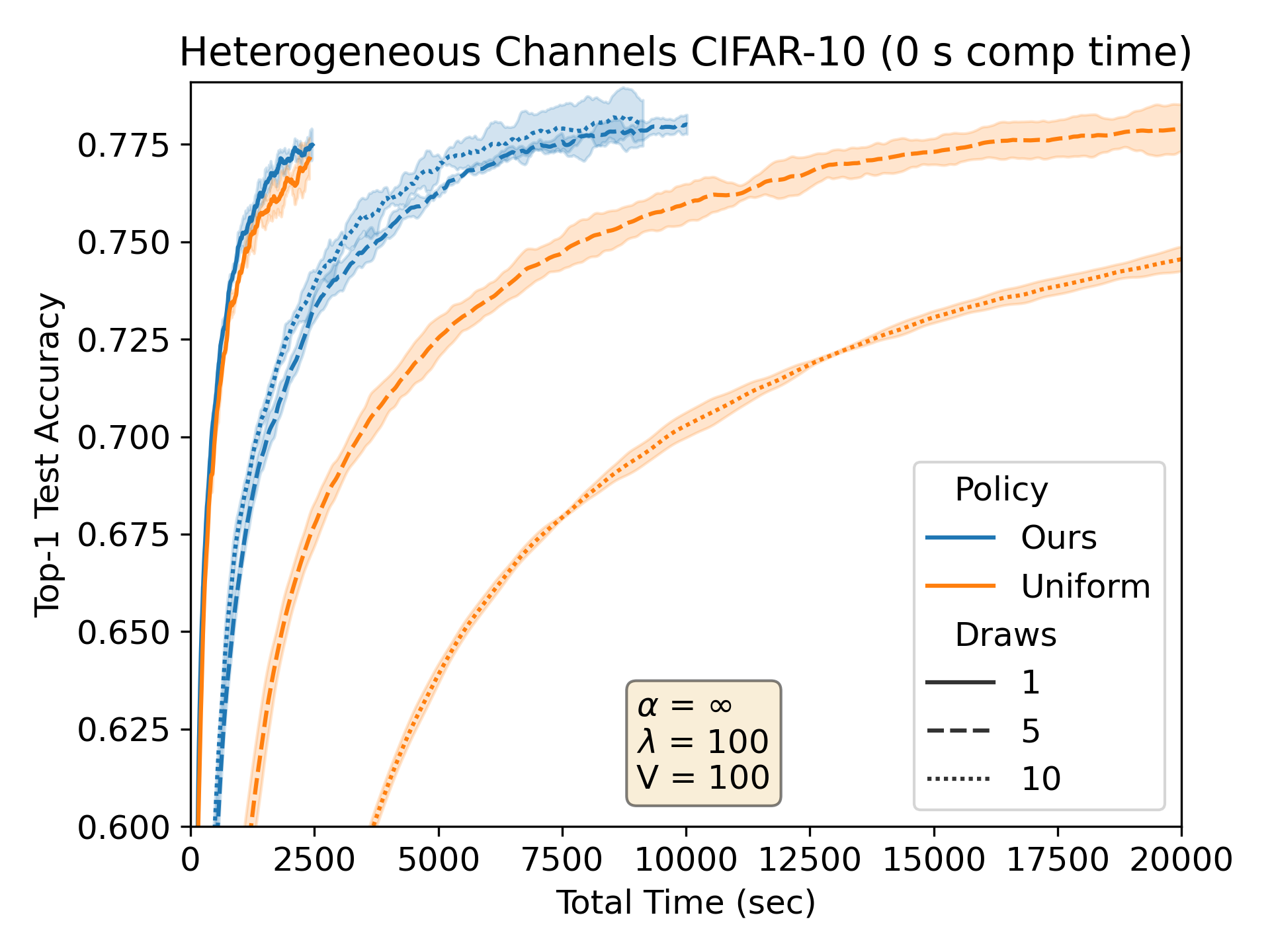}
        \caption{Testing accuracy over time (instantaneous computation).}
        \label{fig:Het_CIFAR10_IID}
    \end{subfigure}
    \hfill
    \begin{subfigure}[b]{0.49\linewidth}
        \centering
        \includegraphics[width=1\linewidth]{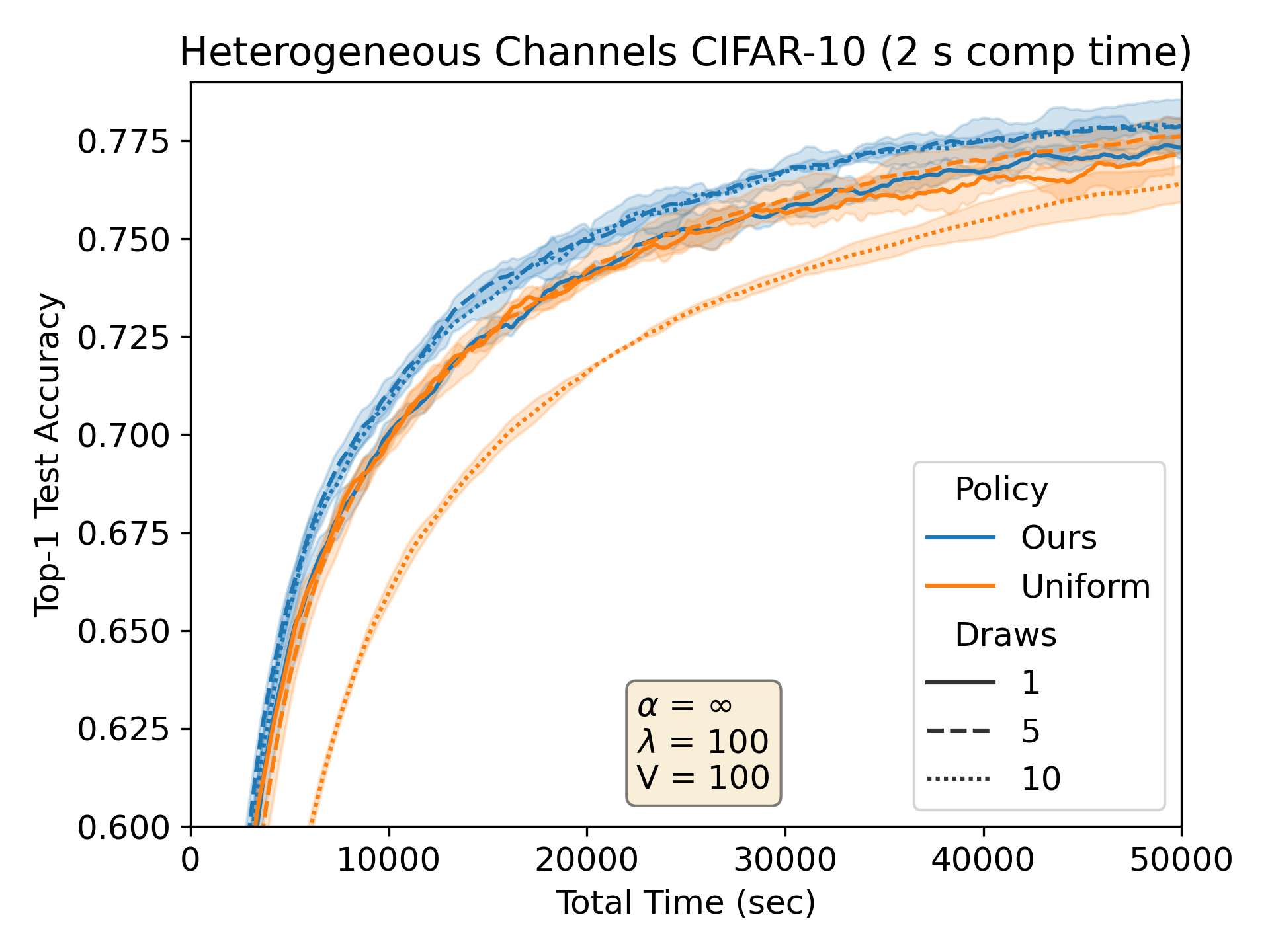}
        \caption{Testing accuracy over time (2 sec parallel computation).}
        \label{fig:Het_CIFAR10_IID2}
    \end{subfigure}

    
    \begin{subfigure}[b]{0.49\linewidth}
        \centering
        \includegraphics[width=1\linewidth]{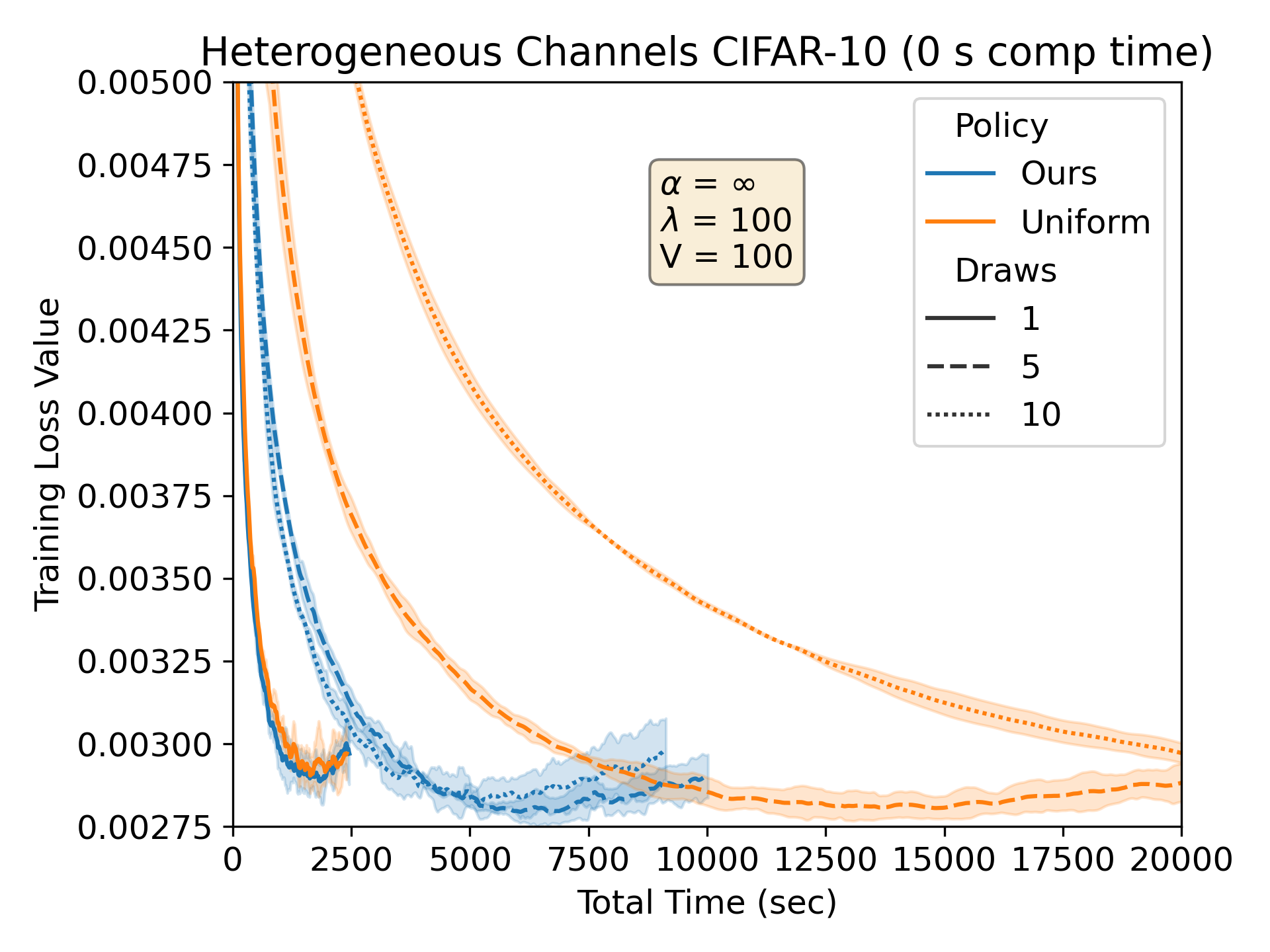}
        \caption{Training loss over time (instantaneous computation).}
        \label{fig:Het_CIFAR10_IID_loss}
    \end{subfigure}
    \hfill
    \begin{subfigure}[b]{0.49\linewidth}
        \centering
        \includegraphics[width=1\linewidth]{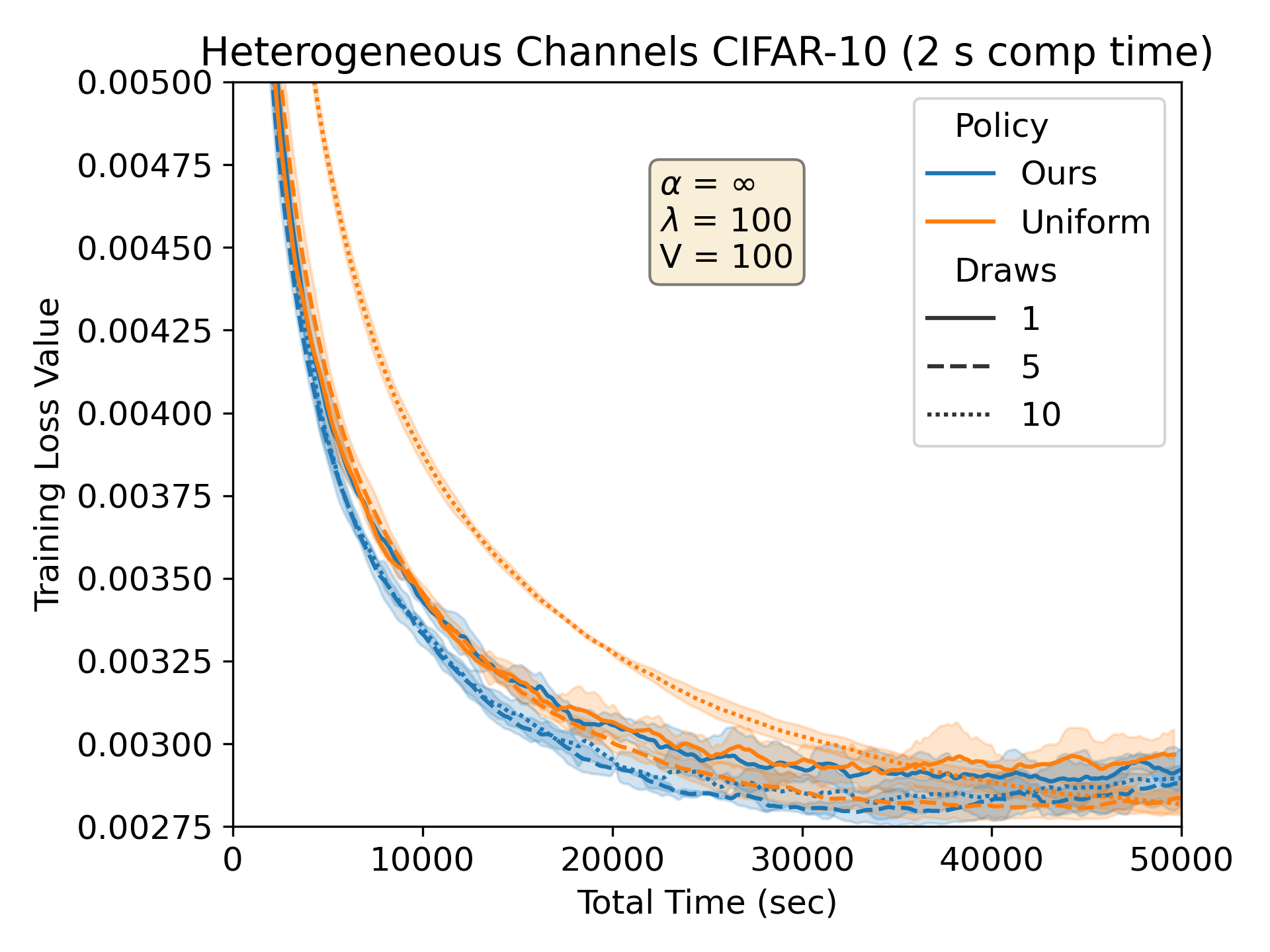}
        \caption{Training loss over time (2 sec parallel computation).}
        \label{fig:Het_CIFAR10_IID_loss2}
    \end{subfigure}
    
    \caption{Comparison of total communication time for uniform selection vs proposed algorithm on CIFAR-10 dataset.}
    \label{fig:CIFAR10}
\end{figure}

Next, in Fig. \ref{fig:Het_CIFAR10_nonIID}, we examine the effect of heterogeneity in the data through parameter $\alpha$.
As $\alpha$ increases, the data distribution across devices becomes more homogeneous such that each device has access to samples of more classes.
For this experiment, we set the computation time to one second, $m=10$, and $\lambda=100$.
We chose a one second computation time since the optimal hyperparameters for all three $\alpha$ levels is the same as seen in Table \ref{tab:opt_config}.
The results again show that our algorithm outperforms the uniform baseline, but that \emph{the advantage is lesser for higher degrees of heterogeneity} in the data.
This intuitively makes sense since a more uniform sampling approach will help avoid client shift in the model as each node will contribute equally on average.
Still, though, our algorithm provides an advantage by avoiding nodes with poor instantaneous channels.
The general trend of slower convergence for higher \noniid-ness also holds.
The convergence speed for $\alpha=0$ is especially bad since each device only has access to samples from one class.

\begin{figure}
    \centering
    \begin{subfigure}[b]{0.475\linewidth}
        \centering
        \includegraphics[width=1\linewidth]{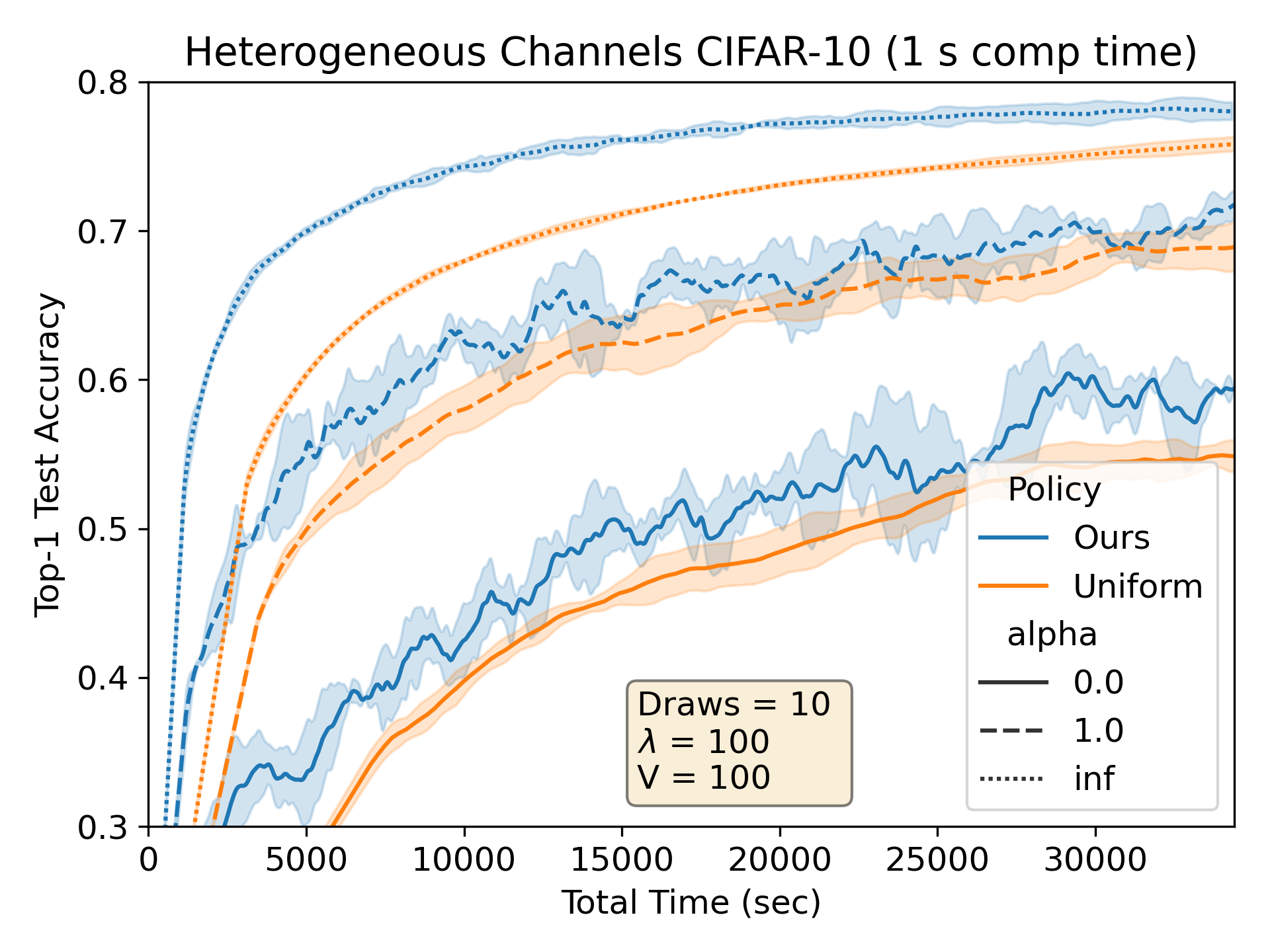}
        \caption{Testing accuracy over time for varying levels of data heterogeneity.}
        \label{fig:Het_CIFAR10_nonIID_acc}
    \end{subfigure}
    \hfill
    \begin{subfigure}[b]{0.475\linewidth}
        \centering
        \includegraphics[width=1\linewidth]{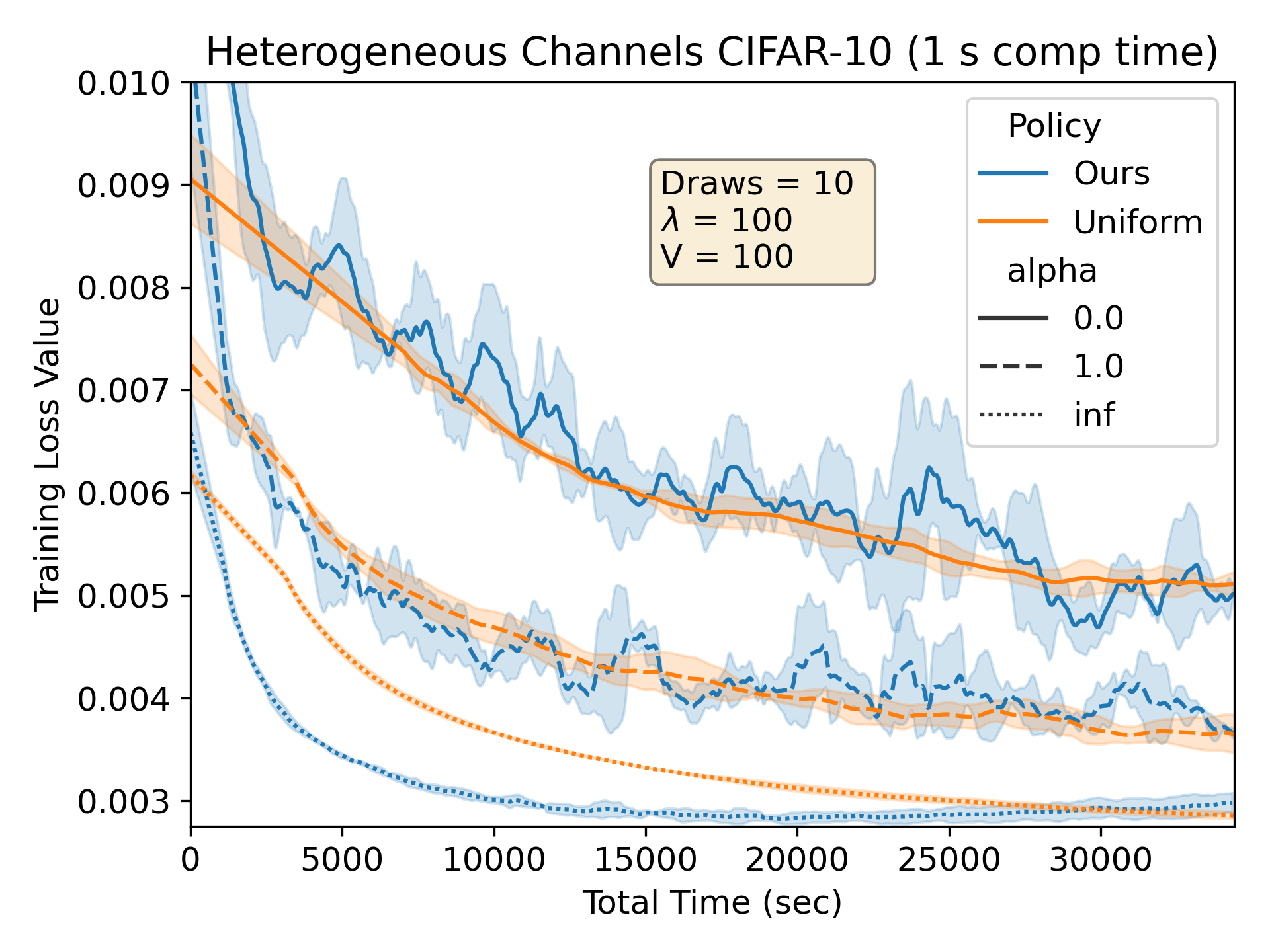}
        \caption{Training loss over time for varying levels of data heterogeneity.}
        \label{fig:Het_CIFAR10_nonIID_loss}
    \end{subfigure}
    \caption{Convergence for heterogeneous data.}
    \label{fig:Het_CIFAR10_nonIID}
\end{figure}

Finally, in Fig. \ref{fig:Het_CIFAR10_IID_iter}, we plot the same results as in Fig. \ref{fig:CIFAR10} but in terms of communication rounds rather than wall-clock time.
This is shown to demonstrate how cases like $m=1$ take more iterations to converge and are slowest in the traditional accuracy per iteration/epoch convergence metric.
However, since these iterations can occur so quickly with less time spent communicating in each round, these cases are much better in practice.
As computation time increases and begins to dominate the total time, though, the trends over total time will begin to match the ones shown here.
Although, we reiterate again that simply setting $m=N=100$ does not guarantee the best performance when computation is high as the generalization of minibatch sizes comes into play \cite{keskar2016large}.

\begin{figure}
    \centering
    \begin{subfigure}[b]{0.49\linewidth}
        \centering
        \includegraphics[width=1\linewidth]{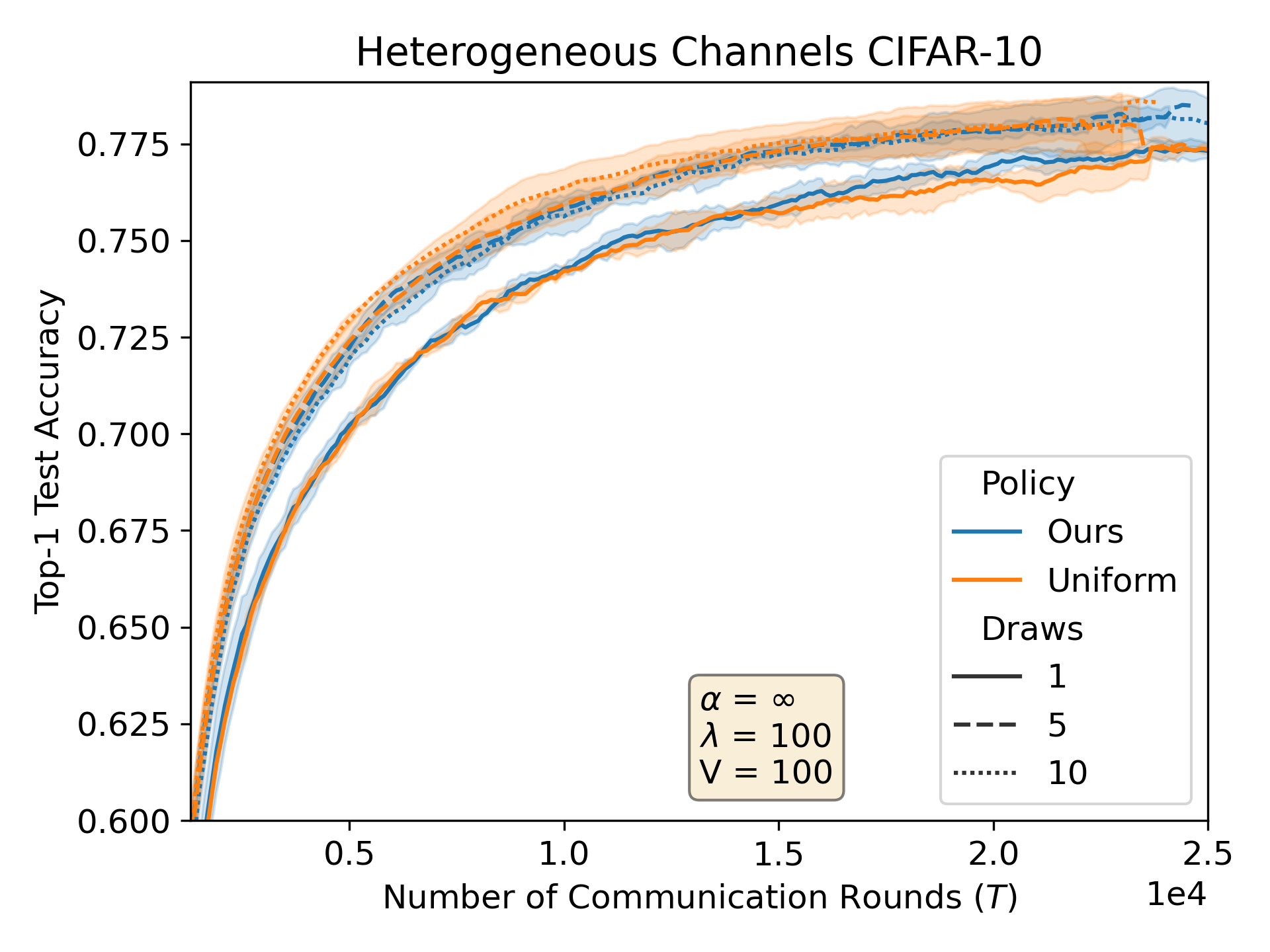}
        \caption{Testing accuracy per communication round ($T$).}
        \label{fig:Het_CIFAR10_IID_iter_acc}
    \end{subfigure}
    \hfill
    \begin{subfigure}[b]{0.49\linewidth}
        \centering
        \includegraphics[width=1\linewidth]{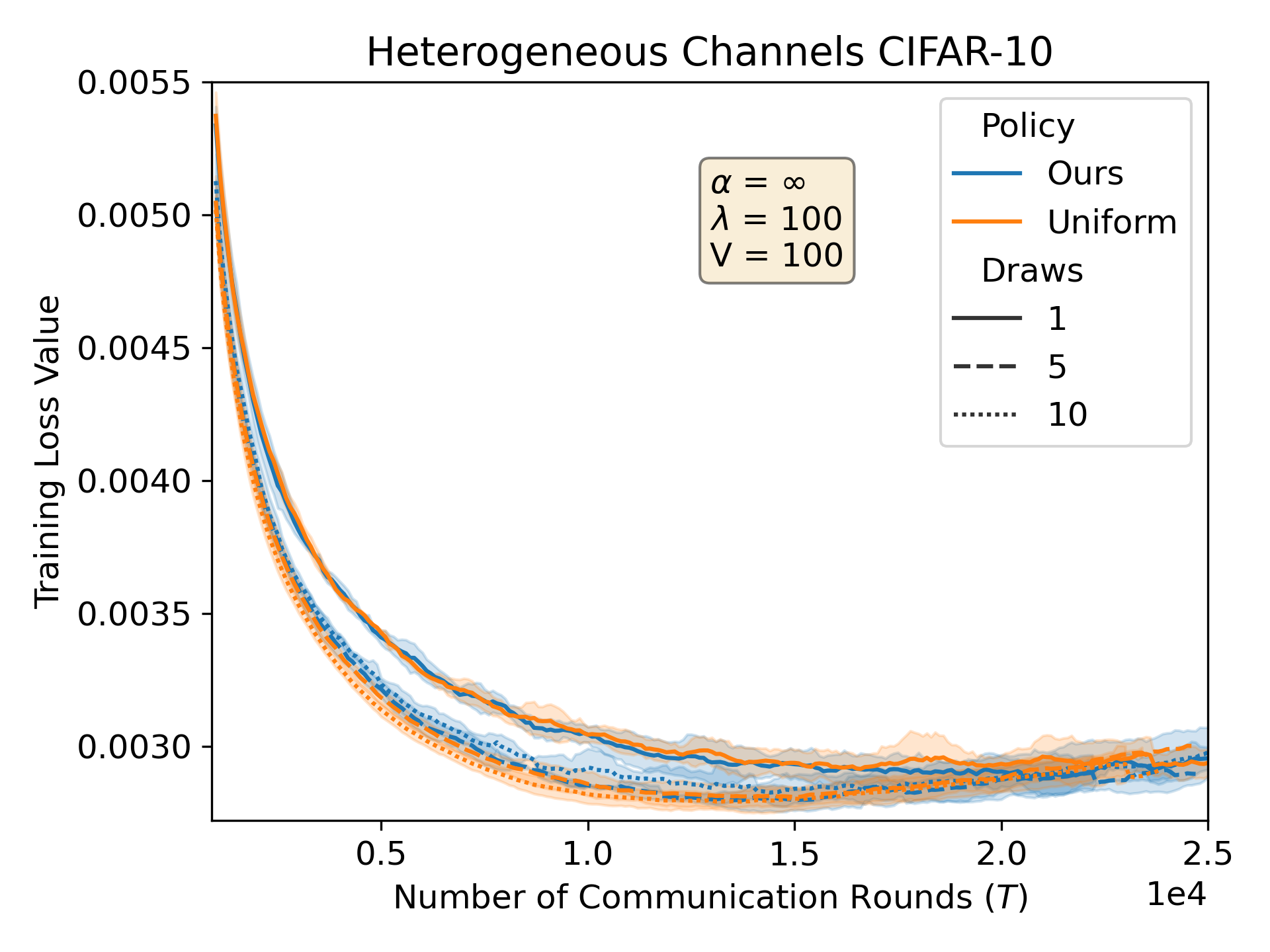}
        \caption{Training loss per communication round ($T$).}
        \label{fig:Het_CIFAR10_IID_iter_loss}
    \end{subfigure}
    \caption{Convergence over communication rounds/iterations.}
    \label{fig:Het_CIFAR10_IID_iter}
\end{figure}

\subsection{Homogeneous Channels}
In this section, we briefly explore the convergence behavior for homogeneous channels where each device has the same Rayleigh fading parameter of $\sigma=1$.
We produce a table of optimal hyperparameters in Table \ref{tab:opt_config_hom} and plot convergence results in Figure \ref{fig:Hom_CIFAR10}.
We plot in terms of top-1 testing accuracy and training loss and for two different levels of data heterogeneity.
The table contains the same trends as in Table \ref{tab:opt_config}.
In Fig. \ref{fig:Hom_CIFAR10}, it can be seen that our algorithm still \emph{outperforms the uniform baseline}, especially for \iid~data, but the speedup is not as significant.
For example, the time to target accuracy of $0.775$ for $\alpha=\infty$ and $m=1$, achieves a $1.24$x speedup.
For fair comparison, we picked $m$ such that both selection policies are optimal for the same value according to Table \ref{tab:opt_config_hom}.
For different levels of $m$, the speedup is more significant since the probability of choosing a slow device increases with more draws when using the uniform sampling policy.
For example, when $m=10$, there is a speed up of $6.7\times$ (not pictured here) over uniform as compared to $8.5\times$ in the heterogeneous channels case. 

\begin{table}[h]
    \centering
    \caption{Best hyperparameters from experiments for different computation time ranges and levels of data heterogeneity.}
    \begin{tabular}{ |c||c|c|c| }
         \hline
         \multicolumn{4}{|c|}{Homogeneous Channel Gain} \\
         \hline
         \multicolumn{1}{|c||}{}&\multicolumn{3}{c|}{Ours}\\
         \hline
         Case& Comp Times (s) & Draws ($m$) & $\lambda$ \\
         \hline
         $\alpha=\infty$, & $0 - 2.48$   &  $1$  & $100$ \\
         target $=0.775$  &   $2.48 - 13.1$  & $5$   & $100$  \\
         & $13.1 - \infty$  & $10$ & $1$  \\
         \hline
         $\alpha=1$, & $0 - 0.69$  & $1$  & $100$ \\
         target $=0.74$ & $0.69 - 10.1$ & $10$ & $100$\\
         & $10.1 - \infty$ & $10$ & $1$ \\
         \hline
         \hline
         \multicolumn{1}{|c||}{}&\multicolumn{3}{c|}{Uniform}\\
         \hline
         Case& Comp Times (s) & Draws ($m$) &  \\
         \hline
         $\alpha=\infty$, & $0 - 2.5$ & $1$ & \\
         target $=0.775$ & $2.5 - \infty$  &  $5$ & \\
         \hline
         $\alpha=1$,  & $0 - 1.3$  &  $1$ &   \\
         target $=0.73$ & $1.8 - 18$  &  $5$ &  \\
         & $18 - \infty$  & $10$ &  \\
         \hline
    \end{tabular}
    \label{tab:opt_config_hom}
\end{table}

Since all devices are equally likely to have good channels in the homogeneous channel case, device selection in our algorithm is more uniform.
Therefore, we expect the impact of data heterogeneity to be lessened as compared to the heterogeneous channel case.
However, this does not appear to be the case.
The reasoning for this is that even though the devices with higher $\sigma$ are chosen more often in the heterogeneous case (see Section \ref{sec:effectV}), all $10$ classes are still present among those top choices.
So, even with skewed sampling, the algorithm is learning from all classes.
If the distribution of classes is dependent on channel quality, however, then the decrease in convergence would be expected to be more pronounced.

\begin{figure}
    \centering
    \begin{subfigure}[b]{0.49\linewidth}
        \centering
        \includegraphics[width=1\linewidth]{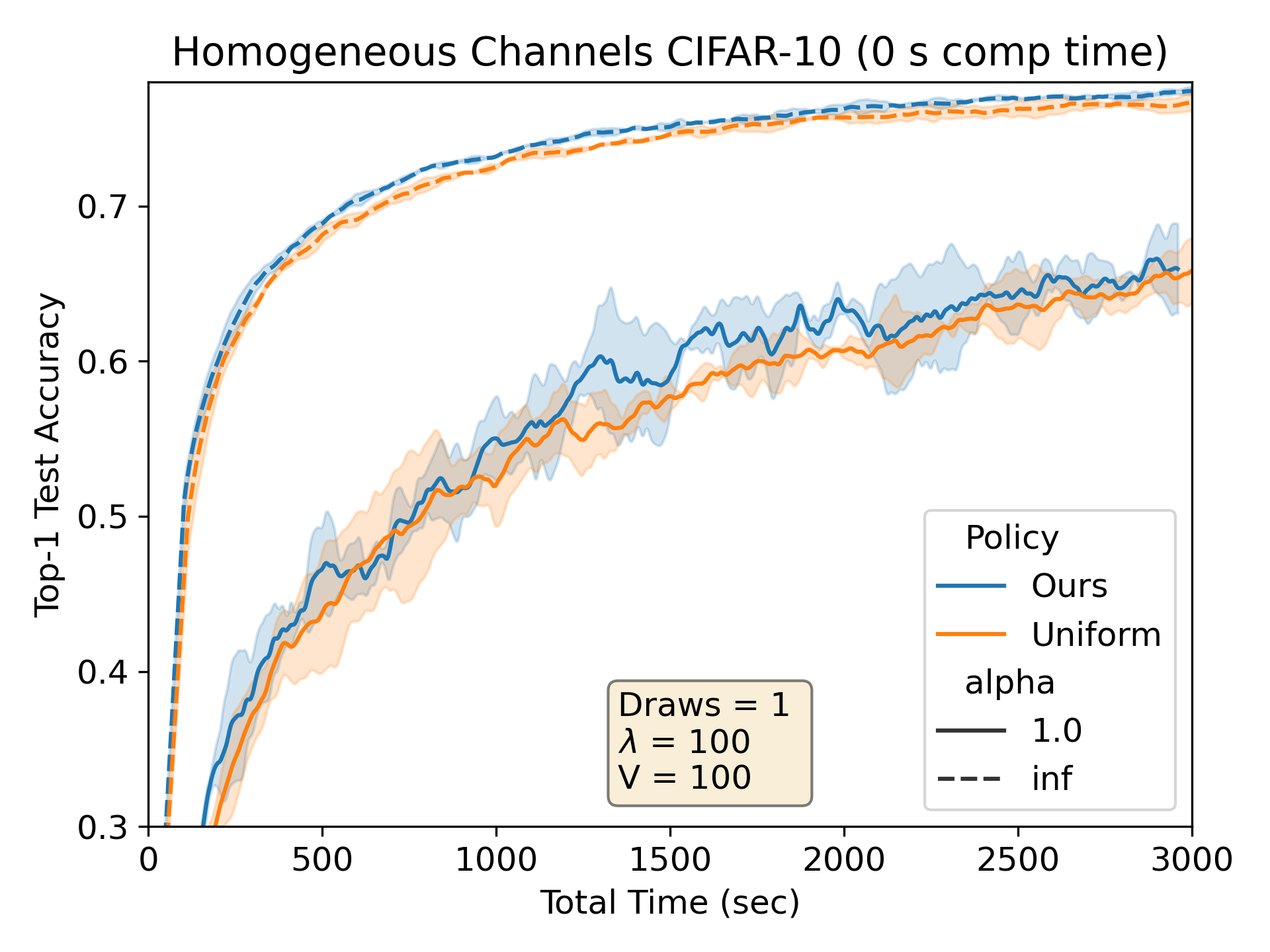}
        \caption{Testing accuracy over time.}
        \label{fig:Hom_CIFAR10_acc}
    \end{subfigure}
    \hfill
    \begin{subfigure}[b]{0.49\linewidth}
        \centering
        \includegraphics[width=1\linewidth]{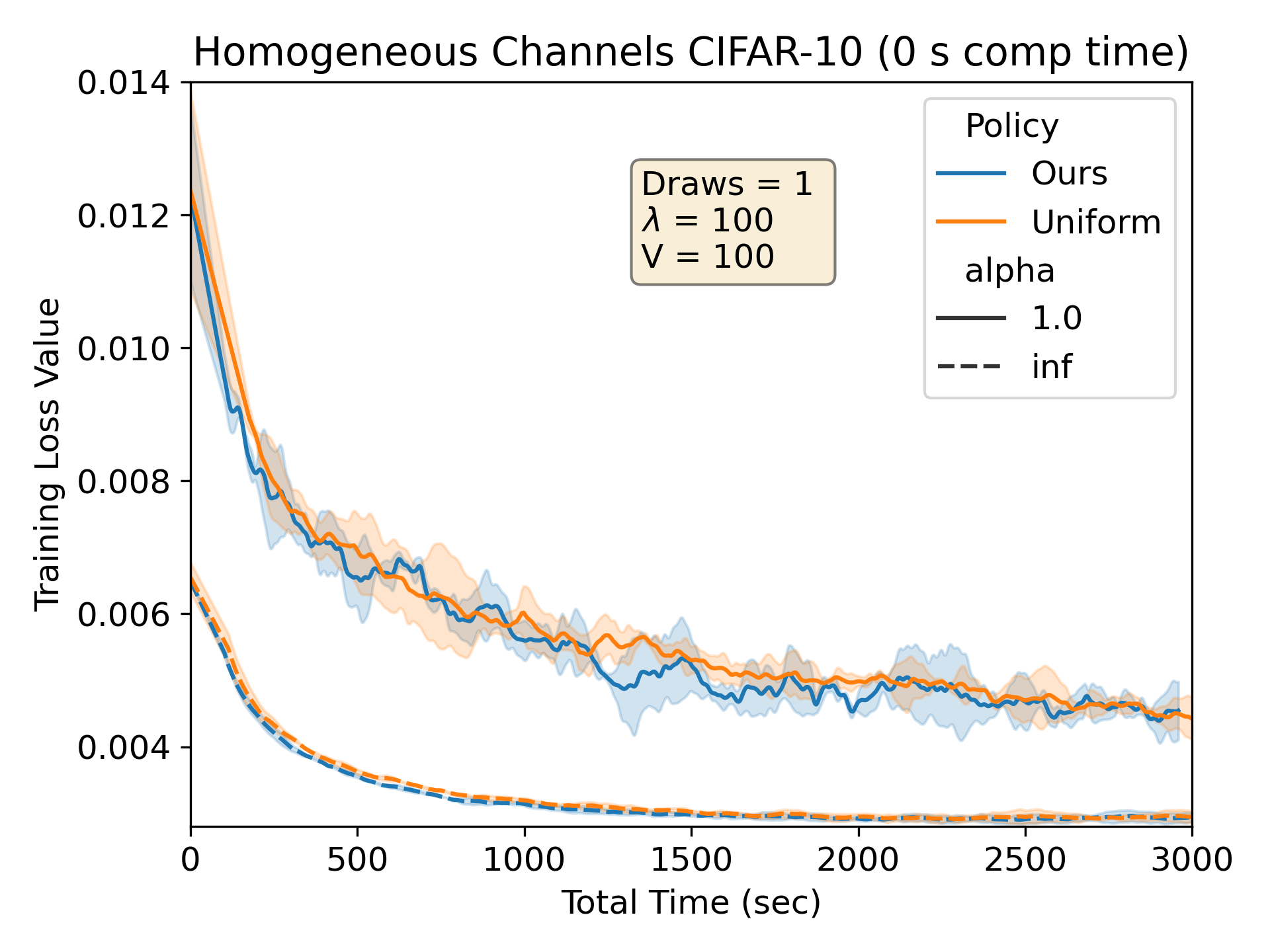}
        \caption{Training loss over time.}
        \label{fig:Hom_CIFAR10_loss}
    \end{subfigure}
    \caption{Convergence for homogeneous channels.}
    \label{fig:Hom_CIFAR10}
\end{figure}

\subsection{Additional Convergence and Statistical Insights}\label{sec:effectV}
In this section, we further explore the behavior of our client sampling algorithm.
First, we show how the choice of parameter $V$ affects convergence of satisfying the time average power constraint.
From \eqref{eqn:optimGap}, we know that $V$ controls the gap in optimality where larger values result in a smaller gap.
Normally, $V$ also conversely controls the queue backlog lengths, but since we only deal with virtual queues, the behavior is slightly different.
In Figure \ref{fig:effectOfV}, we plot the expected time average transmit power $\frac{1}{T} \sum_{t=0}^{T-1} {P_n^t q_n^t}$ over the course of multiple communication rounds.
We see here that increased $V$ causes the constraint of $\bar{P}=1$ to be satisfied at a slower rate.
For example, $V=10^{5}$ does not reach $\bar{P}=1$ within a realistic number of rounds.
This is important because even though it is guaranteed to converge asymptotically, the FL training will have to terminate in finite time in practice.
Similarly for $V=1$, the constraint may be violated at a given finite termination time $T$ due to its oscillatory behavior despite offering the worst optimality gap. 
\emph{Our algorithm sacrifices initial constraint violation in finite time in order to make early gains in performance.}
The improvements over uniform are not solely attributed to this, however.
For the previous experiments, we chose $V=100$ since it satisfies the constraint without extreme advantage over uniform.

\begin{figure}
    \centering
    \includegraphics[width=.6\linewidth]{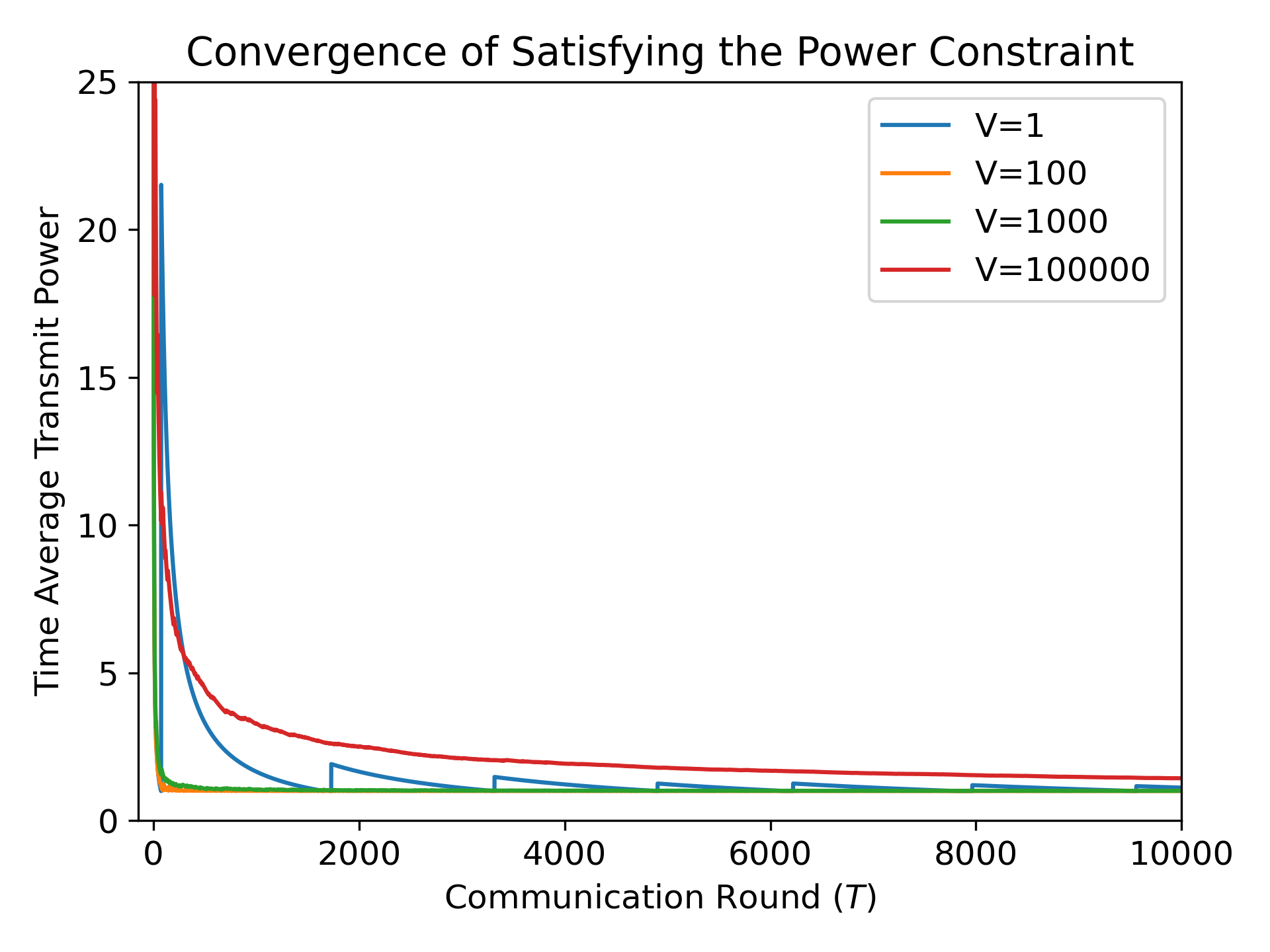}
    \caption{The convergence of the constraint for different values of $V$. The larger the $V$, the more rounds it takes until the constraint is satisfied. Here, the constraint is $\bar{P}_n = 1 $ for all $n$.}
    \label{fig:effectOfV}
\end{figure}

Last, we examine the device selection statistics of our algorithm versus uniform in Fig. \ref{fig:dev_stat}.
Since we mainly plotted the $\lambda=100$ case in the previous sections, we show here how other values affect the device selection process.
In Fig. \ref{fig:dev_freq}, we show a histogram of the selection frequency of all $N=100$ devices in the heterogeneous case for various $\lambda$ values.
The $x$ axis displays each device's Rayleigh fading parameter in increasing order.
When $\lambda$ is larger, minimizing communication time is prioritized over the convergence bound.
Thus, for $\lambda=100$, there is a large skew in selecting the devices with better channels on average.
As $\lambda$ decreases, the selection becomes more uniform where $\lambda=1$ and uniform have similar distributions.
Next, in Figure \ref{fig:dev_per_round}, we investigate how $\lambda$ affects the number of devices chosen per round when using a sampling \emph{with} replacement policy.
Since a device can be drawn more than once, the actual number of devices scheduled in a given round may not equal $m$, even in the uniform case.
Here, you can see that lower $\lambda$ results in probabilities closer to uniform, whereas when $\lambda=100$, it is much more common to schedule only $3$ to $5$ devices despite drawing $m=10$ times.
Through the selection probabilities, \emph{our algorithm is able to adaptively change how many devices are chosen depending on current channel conditions}.

\begin{figure}
    \centering
    \begin{subfigure}[t]{0.492\linewidth}
    \centering
        \includegraphics[width=\linewidth]{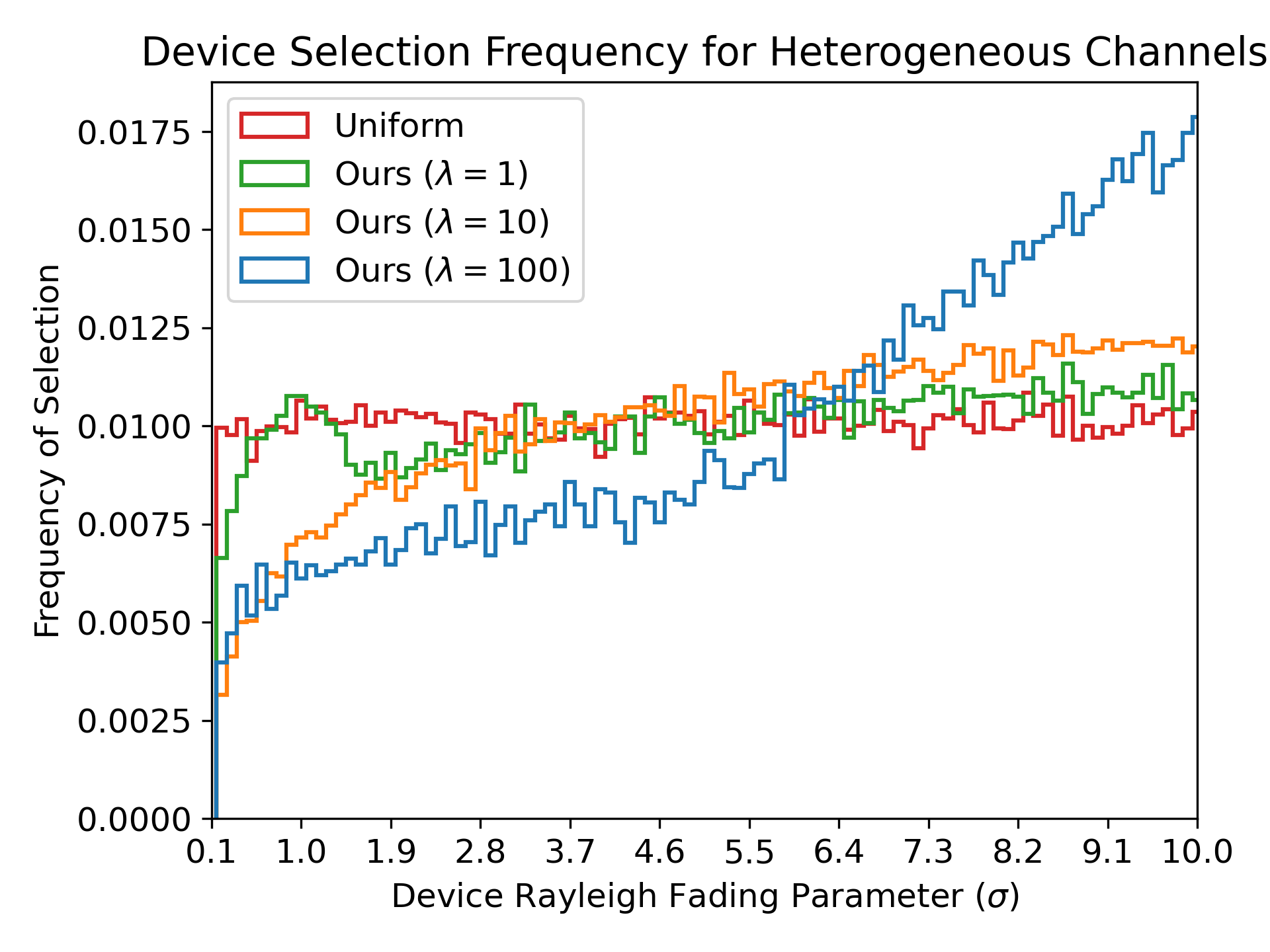}
        \subcaption{The frequency of selection for each device for different $\lambda$ and policy. Each device's Rayleigh fading parameter is shown on the x-axis.}
        \label{fig:dev_freq}
    \end{subfigure}
    \hfill
    \begin{subfigure}[t]{0.492\linewidth}
        \centering
        \includegraphics[width=\linewidth]{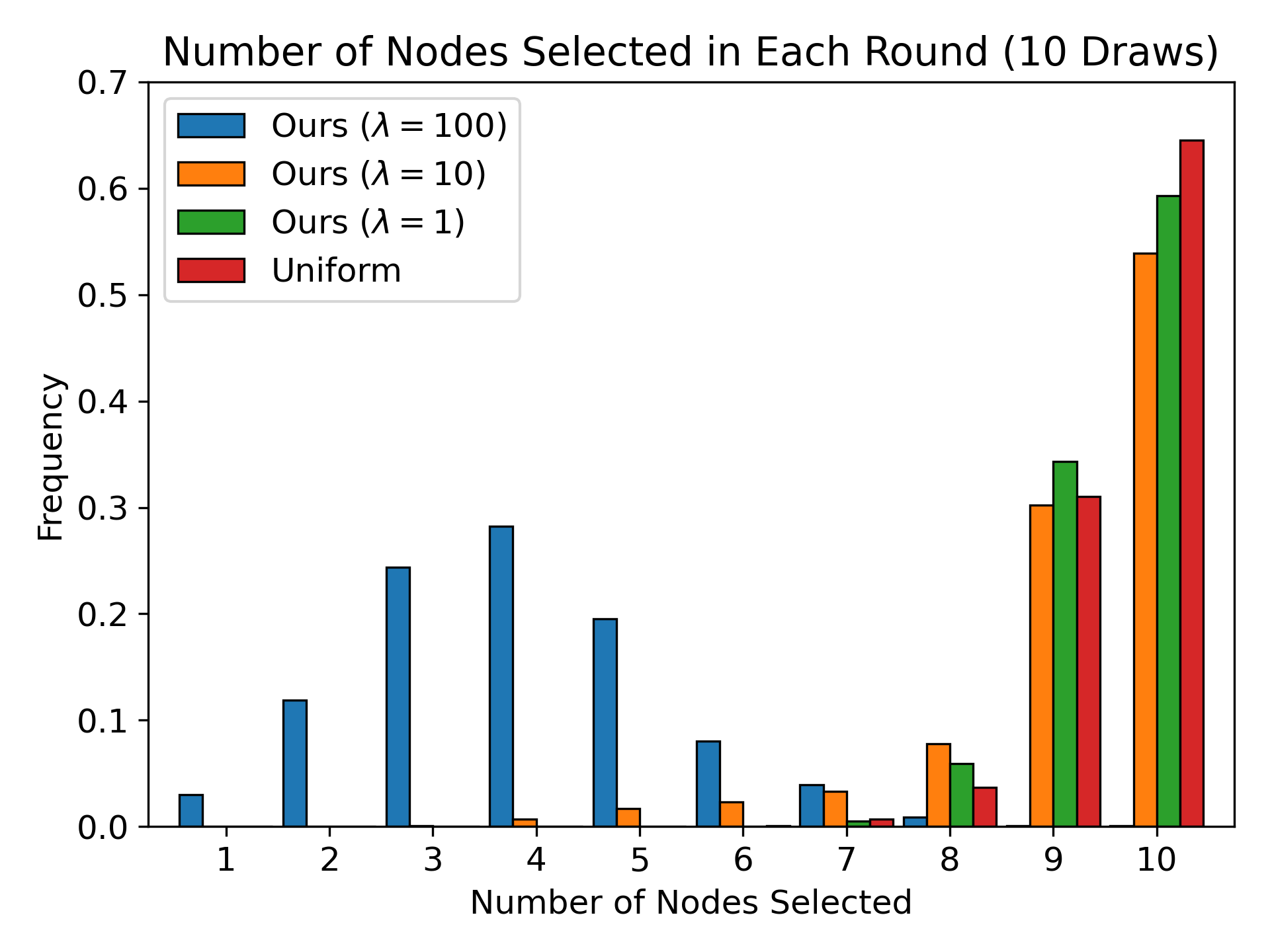}
        \subcaption{The probability of the number of devices chosen per round via sampling \emph{with} replacement for varying $\lambda$ and policy.}
        \label{fig:dev_per_round}
    \end{subfigure}
    \caption{Device selection statistics.}
    \label{fig:dev_stat}
\end{figure}

%% file: Conclusion.tex
\section{Conclusion}
In this paper, we explored the impact of device scheduling over wireless links in federated learning and proposed a novel scheduling algorithm that leverages stochastic optimization theory. 
Our experiments saw up to $8.5$x speedup over the uniform baselines in some cases.
The first contribution of our work was developing a convergence bound for non-convex FL loss functions with arbitrary device participation probabilities without a bounded gradient assumption.
We showed that it recovers the state-of-the-art convergence rate with linear speedup.
We then used the bound to formulate a time average optimization problem with a time average power constraint that can be solved using the Lyapunov drift-plus-penalty framework.
The resulting online algorithm can be solved at each communication round without requiring knowledge of channel statistics.
Our experiments show impressive performance over the uniform baseline in terms of wall-clock time.
We explored many scenarios such varying levels of data and channel heterogeneity as well as varying computation time and hyperparameters.
Depending on the scenario, we showed that it is sometimes better to perform quicker, less informative updates than slower, more informative updates as seen in the fantastic performance of the $m=1$ regime for low computation time.
There are very interesting trade-offs to consider when the cost of each training iteration is variable which is the subject of future research.


%% file: Appendix.tex
\appendix

We use the following inequalities throughout the convergence proof.\\
\textbf{Preliminary (in)equalities.}
We will use Jensen's inequality:
\begin{align}
\begin{array}{cc}
     \normsq{ \frac{1}{M}\sum_{m=1}^M \y_m } \leq \frac{1}{M}\sum_{m=1}^M \normsq{\y_m }
\end{array}
     \\
\begin{array}{cc}
    \normsq{ \sum_{m=1}^M \y_m } &\leq M \sum_{m=1}^M \normsq{ \y_m }
\end{array},
\end{align}
and the Peter-Paul inequality:
\begin{equation}
    \langle \y_1, \y_2 \rangle \leq \frac{\rho \normsq{ \y_1 }}{2} + \frac{\normsq{ \y_2 }}{2\rho}
\end{equation}
for $\rho > 0$.
We also use
\begin{align}\label{eq:innerprod}
    \innerprod{\x,\y} = \frac{1}{2}\left(\normsq{\x}+\normsq{\y}-\normsq{\x-\y}\right)
\end{align}

\subsection{Proof of Theorem \ref{thm:1}}\label{sec:appendix}
\input{proof1v4}

\begin{lemma}[{\cite[Lemma 3]{reddi2020adaptive}}] \label{lm:lemma2}
    If $\gamma \leq \frac{1}{8LK}$, then
    \begin{align}
        &\Expectt{\normsq{ \y^n_{t,i} - \x_t}} \nonumber \\
        &\qquad \leq 5K\gamma^2\left(\nu^2+6K\epsilon^2 \right) + 30K^2\gamma^2 \normsq{\nabla  f(\x_t)}. \label{eqn:gradBound2}
    \end{align}
\end{lemma}
\begin{proof}
    This lemma is a slight variation of Lemma 3 in \cite{reddi2020adaptive} where instead of taking the arithmetic average over all $N$ clients, we instead bound uniformly across all $n$.
    This change is straightforward, so we omit details for space.
\end{proof}

\subsection{Proof of Theorem \ref{thm:Lyapunov}}\label{sec:appendix2}
\input{proof2}

%% file: proof1v4.tex
First, we note that via Algorithm \ref{alg:fedavg}, we have
\begin{align*}
    \x_{t+1}-\x_t 
    &= \frac{1}{N}\sum_{n=1}^N \frac{\Identity_n^t}{q_n^t} (\y^n_{t,K} - \y^n_{t,0})\\
    &= - \frac{\gamma}{N}\sum_{n=1}^N \frac{\Identity_n^t}{q_n^t} \sum_{i=0}^{K-1}  \g_n(\y^n_{t,i}) .
\end{align*}
Let $\Expect_t[\cdot] := \Expect[\cdot|x_t,\{q_n^t\}]$, which denotes the expectation over the randomness of SGD \emph{and} client sampling.
From $L$-smoothness, we have
\begin{align}
    &\Expectt{f(\x_{t+1})} \nonumber\\
    &\leq f(\x_t) +\innerprod{\nabla f(\x_t), \Expectt{\x_{t+1} - \x_{t}}} \nonumber \\
    &\qquad + \frac{L}{2} \Expectt{\normsq{\x_{t+1} - \x_{t}}}\nonumber\\
    &= f(\x_t) - \innerprod{\nabla f(\x_t), \Expectt{\frac{1}{N}\sum_{n=1}^N \frac{\Identity_n^t}{q_n^t} \sum_{i=0}^{K-1}  \gamma\g_n(\y^n_{t,i})}}\nonumber\\
    & \qquad + \frac{L\gamma^2}{2N^2} \Expectt{\normsq{\sum_{n=1}^N \frac{\Identity_n^t}{q_n^t} \sum_{i=0}^{K-1}  \g_n(\y^n_{t,i})}}\nonumber\\
    &\overset{(a)}{=} f(\x_t) - \innerprod{\nabla f(\x_t), \Expectt{\frac{\gamma}{N}\sum_{n=1}^N  \sum_{i=0}^{K-1}  \nabla f_n(\y^n_{t,i})}}\nonumber \\
    & \qquad + \frac{L\gamma^2}{2N^2} \Expectt{\normsq{\sum_{n=1}^N \frac{\Identity_n^t}{q_n^t} \sum_{i=0}^{K-1}  \g_n(\y^n_{t,i})}},
    \label{eq:convergence_proof_1}
\end{align}
where (a) uses the independence between $\Identity_n^t$ and $\g_n$, the fact that $\Expectcond{\Identity_n^t}{\x_t} = \Expectbracket{\Identity_n^t} = q_n^t$, and the total expectation $\Expectcond{\g_n(\y^n_{t,i})}{\x_t} = \Expectcond{\Expectcond{\g_n(\y^n_{t,i})}{\y^n_{t,i}, \x_t}}{\x_t} =  \Expectcond{\nabla f_n(\y^n_{t,i})}{\x_t}$.

We now bound the terms in \eqref{eq:convergence_proof_1} separately.
We begin with the second term and note that
\begin{align*}
    &-\innerprod{\nabla f(\x_t),\Expectt{\frac{\gamma}{N}\sum_{n=1}^N\sum_{i=0}^{K-1}\nabla f_n(\y^n_{t,i})}} \\
    & \ = -\left\langle \nabla f(\x_t), \gamma\Expectt{\begin{array}{c}\frac{1}{N}\sum_{n=1}^N\sum_{i=0}^{K-1}  \nabla f_n(\y^n_{t,i}) \\ - K\nabla f(\x_t) + K\nabla f(\x_t)\end{array}}  \right\rangle \\
    &\ = \innerprod{\nabla f(\x_t), -\gamma \Expectt{\frac{1}{N}\sum_{n=1}^N\sum_{i=0}^{K-1}\nabla f_n(\y^n_{t,i}) \!-\! K \nabla f(\x_t)  }} \\
    & \qquad - \gamma K \normsq{\nabla f(\x_t)} \\
    &\ = \innerprod{\!\sqrt{\gamma K}\nabla f(\x_t), - \frac{\sqrt{\gamma}}{N \sqrt{K}} \Expectt{\sum_{n=1}^N\sum_{i=0}^{K-1}  \!\left({\setlength\arraycolsep{-1pt}\begin{array}{l}\nabla f_n(\y^n_{t,i}) \\ \ - \nabla f_n(\x_t) \end{array}}\right) }} \\
    & \qquad- \gamma K \normsq{\nabla f(\x_t)} \\
    &\ \overset{(a)}{=} \frac{\gamma}{2 N^2 K}\Expectt{\normsq{\sum_{n=1}^N\sum_{i=0}^{K-1}\left(\nabla f_n(\y^n_{t,i}) - \nabla f_n(\x_t)\right) }} \\
    & \qquad - \frac{\gamma}{2 N^2 K}\Expectt{\normsq{\sum_{n=1}^N\sum_{i=0}^{K-1}\nabla f_n(\y^n_{t,i})}} \\
    & \qquad -\frac{\gamma K}{2}\normsq{\nabla f(\x_t)} \\
    &\  \overset{(b)}{\leq} \frac{\gamma K L^2}{2}\left( 5K\gamma^2\left(\nu^2+6K\epsilon^2 \right) + 30K^2\gamma^2 \normsq{\nabla  f(\x_t)} \right)\\
    & \qquad - \frac{\gamma}{2 N^2 K}\Expectt{\normsq{\sum_{n=1}^N\sum_{i=0}^{K-1}\nabla f_n(\y^n_{t,i})}} \\
    & \qquad - \frac{\gamma K}{2}\normsq{\nabla f(\x_t)} \\
    &\ = \frac{ 5 \gamma^3 K^2 L^2}{2}\left(\nu^2+6K\epsilon^2 \right) \\ 
    & \qquad + \left(15\gamma^3 K^3 L^2 -\frac{\gamma K}{2}\right) \normsq{\nabla f(\x_t)}\\
    & \qquad - \frac{\gamma}{2 N^2 K}\Expectt{\normsq{\sum_{n=1}^N\sum_{i=0}^{K-1}\nabla f_n(\y^n_{t,i})}},
\end{align*}
where (a) is due to \eqref{eq:innerprod} and (b) is from
\begin{align}
    &\Expectt{\normsq{\sum_{n=1}^N\sum_{i=0}^{K-1}\left(\nabla f_n(\y^n_{t,i}) - \nabla f_n(\x_t)\right) }} \nonumber\\
    &\qquad \leq NK\sum_{n=1}^N\sum_{i=0}^{K-1}\Expectt{\normsq{\left(\nabla f_n(\y^n_{t,i}) - \nabla f_n(\x_t)\right) }} \nonumber\\
    &\qquad \leq NKL^2\sum_{n=1}^N\sum_{i=0}^{K-1} \Expectt{\normsq{\x_t - \y^n_{t,i}}} \nonumber \\
    &\qquad \leq N^2K^2L^2 \left({\setlength\arraycolsep{-1pt}\begin{array}{l}5K\gamma^2\left(\nu^2+6K\epsilon^2 \right) \\ \qquad\qquad + \, 30K^2\gamma^2 \normsq{\nabla  f(\x_t)} \end{array}}\right)
\end{align}
which utilizes Lemma \ref{lm:lemma2} for the final inequality.

Next, we bound the third term in \eqref{eq:convergence_proof_1} as
\begin{align*}
    &\Expectt{\normsq{\sum_{n=1}^N \frac{\Identity_n^t}{q_n^t} \sum_{i=0}^{K-1}  \g_n(\y^n_{t,i})}} \\
    & \ = \Expectt{\normsq{\sum_{n=1}^N \sum_{i=0}^{K-1} \frac{\Identity_n^t}{q_n^t}  \left(\g_n(\y^n_{t,i}) \!-\! \nabla f_n(\y^n_{t,i}) \!+\! \nabla f_n(\y^n_{t,i})\right)}} \\
    & \ \leq 2\Expectt{\normsq{\sum_{n=1}^N \sum_{i=0}^{K-1} \frac{\Identity_n^t}{q_n^t}  \nabla f_n(\y^n_{t,i})}} \\
    &\qquad + 2\Expectt{\normsq{\sum_{n=1}^N \sum_{i=0}^{K-1} \frac{\Identity_n^t}{q_n^t}  \left(\g_n(\y^n_{t,i}) \!-\! \nabla f_n(\y^n_{t,i})\right)}} \\ 
    & \ \overset{(a)}{=} 2\Expectt{\normsq{\sum_{n=1}^N \sum_{i=0}^{K-1} \frac{\Identity_n^t}{q_n^t}  \nabla f_n(\y^n_{t,i})}} \\
    &\qquad + 2\sum_{n=1}^N\sum_{i=0}^{K-1}\Expectt{\normsq{  \frac{\Identity_n^t}{q_n^t}  \left(\g_n(\y^n_{t,i}) \!-\! \nabla f_n(\y^n_{t,i})\right)}} \\ 
    & \ = 2\Expectt{\normsq{\sum_{n=1}^N \sum_{i=0}^{K-1} \frac{\Identity_n^t}{q_n^t} \nabla f_n(\y^n_{t,i})}}\\
    &\qquad+ 2\sum_{n=1}^N \sum_{i=0}^{K-1}\frac{1}{q_n^t} \Expectt{\normsq{\g_n(\y^n_{t,i}) - \nabla f_n(\y^n_{t,i})}} \\
    & \ \leq  2\Expectt{\normsq{\sum_{n=1}^N \sum_{i=0}^{K-1} \frac{\Identity_n^t}{q_n^t} \nabla f_n(\y^n_{t,i})}} + 2 N K Q_t \nu^2,
\end{align*}
where (a) is due to the fact that $\{\g_n(\y^n_{t,i}) - \nabla f_n(\y^n_{t,i})\}$ is independent over $n$ and is a martingale difference sequence over $i$ such that $\Expectt{\normsq{\sum \x}} = \sum\Expectt{\normsq{\x}} $, \cite[Lemma 4]{karimireddy2020scaffold}.
The final inequality is from Assumption \ref{asmp:boundedGradNoise} and we define $Q_t = \frac{1}{N}\sum_{n=1}^N \frac{1}{q_n^t}$.

Before returning to \eqref{eq:convergence_proof_1}, we note a few more inequalities.
Let $\z_n = \sum_{i=0}^{K-1}  \nabla f_n(\y^n_{t,i})$ and suppose that there exists a $\qmin \leq q_n^t$ for all $n,t$.
Then, we have
\begin{align}
   \Expectt{\normsq{\sum_{n=1}^N \frac{\Identity_n^t}{q_n^t} \z_n}} &= \sum_{i, j}^N \frac{1}{q_i^t q_j^t}\Expectt{\Identity_i^t\Identity_j^t}\Expectt{\innerprod{ \z_i, \z_j}} \nonumber\\
    &\leq \sum_{i, j}^N \frac{1}{\sqrt{q_i^t q_j^t}} \Expectt{\innerprod{ \z_i, \z_j}} \nonumber \\
    &\leq \sum_{i, j}^N \frac{1}{\qmin} \Expectt{\innerprod{ \z_i, \z_j}} \nonumber\\
    &= \frac{1}{\qmin}\Expectt{\normsq{\sum_{n=1}^N \z_n}} \label{eq:tn1_dep}
\end{align}
where we use the fact that sampling and gradient estimation are independent, and we use Cauchy-Schwarz inequality such that
\begin{align*}
    \Expectt{\Identity_i^t\Identity_j^t} &\leq\sqrt{\Expectt{\left(\Identity_i^t\right)^2}\Expectt{\left(\Identity_j^t\right)^2}}\\
    &=\sqrt{\Expectt{\Identity_i^t}\Expectt{\Identity_j^t}}\\
    &= \sqrt{q_i^t q_j^t}.
\end{align*}

Returning to \eqref{eq:convergence_proof_1} and applying the previous bounds, we have
\begin{align}
    &\Expectt{f(\x_{t+1})} \nonumber \\
    &\quad \leq f(\x_t) - \innerprod{\nabla f(\x_t), \Expectt{\frac{\gamma}{N}\sum_{n=1}^N  \sum_{i=0}^{K-1}  \nabla f_n(\y^n_{t,i})}}\nonumber\\
    &\qquad\quad + \frac{L\gamma^2}{2N^2} \Expectt{\normsq{\sum_{n=1}^N \frac{\Identity_n^t}{q_n^t} \sum_{i=0}^{K-1}  \g_n(\y^n_{t,i})}}\nonumber\\
    &\quad\leq f(\x_t) + \frac{ 5 \gamma^3 K^2 L^2}{2}\left(\nu^2+6K\epsilon^2 \right)\nonumber\\
    &\qquad\quad - \frac{\gamma}{2 N^2 K}\Expectt{\normsq{\sum_{n=1}^N\sum_{i=0}^{K-1}\nabla f_n(\y^n_{t,i})}}\nonumber\\
    &\qquad\quad + \left(15\gamma^3 K^3 L^2 -\frac{\gamma K}{2}\right) \normsq{\nabla f(\x_t)} + \frac{L\gamma^2 K Q_t}{N}\nu^2 \nonumber \\
    &\qquad\quad + \frac{L\gamma^2}{N^2}\left(\Expectt{\normsq{\displaystyle\sum_{n=1}^N \sum_{i=0}^{K-1} \frac{\Identity_n^t}{q_n^t} \nabla f_n(\y^n_{t,i})}} \right) \nonumber \\
    & \quad \leq f(\x_t) + \frac{ 5 \gamma^3 K^2 L^2}{2}\left(\nu^2+6K\epsilon^2 \right) + \frac{L\gamma^2 K Q_t}{N}\nu^2 \nonumber\\
    & \qquad\quad+ \left(15\gamma^3 K^3 L^2 -\frac{\gamma K}{2}\right) \normsq{\nabla f(\x_t)}, \label{eq:penultimate_1}
\end{align}
where the last inequality uses \eqref{eq:tn1_dep} such that
\begin{align}
    &\frac{L\gamma^2}{N^2}\Expectt{\normsq{\displaystyle\sum_{n=1}^N \frac{\Identity_n^t}{q_n^t} \z_n }} - \frac{\gamma}{2 N^2 K}\Expectt{\normsq{\sum_{n=1}^N \z_n}} \nonumber\\
    & \quad \leq \frac{\gamma}{N^2} \left(\frac{L\gamma}{\qmin} - \frac{1}{2K}\right) \Expectt{\normsq{\sum_{n=1}^N \z_n}} \nonumber \\
    &\quad = 0 ,
\end{align}
if $\gamma < \frac{\qmin}{2LK}$.

Continuing from \eqref{eq:penultimate_1}, we have
\begin{align}
    &\Expectt{f(\x_{t+1})} \nonumber \\
    & \quad \leq f(\x_t) + \frac{ 5 \gamma^3 K^2 L^2}{2}\left(\nu^2+6K\epsilon^2 \right) + \frac{L\gamma^2 K Q_t}{N}\nu^2 \nonumber\\
    & \qquad\qquad+ \left(15\gamma^3 K^3 L^2 -\frac{\gamma K}{2}\right) \normsq{\nabla f(\x_t)},  \nonumber \\
    & \quad \leq f(\x_t) + \frac{ 5 \gamma^3 K^2 L^2}{2}\left(\nu^2+6K\epsilon^2 \right) + \frac{L\gamma^2 K Q_t}{N}\nu^2 \nonumber\\
    & \qquad\qquad-c\frac{\gamma K}{2} \normsq{\nabla f(\x_t)}, \label{eq:penultimate}
\end{align}
where the final inequality holds since there exists a constant $c>0$ such that $\left(1 - 30\gamma^2 K^2 L^2 \right) > c > 0$ if we assume that $\gamma < \frac{1}{6 LK}$.

Finally, taking total expectation and then rearranging and summing $t$ from $0$ to $T-1$, we have
\begin{align}
    \frac{1}{T}\sum_{t=0}^{T-1}\Expectbracket{\normsq{\nabla f(\x_t)}} \leq & \frac{2\left(\Expectbracket{f(\x_0)} - \Expectbracket{f(\x_{T})}\right)}{c\gamma KT} \nonumber\\
    &\quad + \Phi_1 +  \frac{\Phi_2}{TN} \sum_{t=0}^{T-1} Q_t \, , \label{eqn:thm1Bound}
\end{align}
where $\Phi_1 = \frac{1}{c}5 \gamma^2 K L^2\left(\nu^2+6K\epsilon^2 \right)$ and $\Phi_2 = \frac{2L\gamma\nu^2}{c}$.

%% file: proof2.tex
Since the objective function is an independent sum over $n$ and has independent boundary constraints for $P_n^t$, we can find the minimizing values $P_n^t$ by finding the roots of the gradient of the objective function and ensuring that they are within the upper and lower bounds.
If no roots are within that set, one of the end points will minimize the function, so we only need to check those points.

To find the roots, we compute the gradient of the objective function for each $n$ in \eqref{eqn:minimization}
\begin{align}\label{eqn:gradient}
    &\nabla f(q_n^t,P_n^t) = \nonumber \\
    &\begin{bmatrix}
    -\frac{V}{N (q_n^t)^2} + \frac{ V\lambda\ell}{B\log_2\left(1+|h_n^t|^2 \frac{P_n^t}{N_0}\right) }+Z_n^tP_n^t\\
     \frac{-V\lambda\ell|h_n^t|^2}{N_0 B\left(1+|h_n^t|^2 \frac{P_n^t}{N_0}\right)\left(\log_2\left(1+|h_n^t|^2 \frac{P_n^t}{N_0}\right)\right)^2} q_n^t +Z_n^t q_n^t
    \end{bmatrix}.
\end{align}
We first look at the partial derivative with respect to $P_n^t$ and note that setting it equal to zero and dividing by $q_n^t$ gives
\begin{align*}
    0=&\nonumber\\
    &\frac{-V\lambda\ell |h_n^t|^2/(N_0 B)}{\left(1+|h_n^t|^2 \frac{P_n^t}{N_0}\right)\left(\log_2\left(1+|h_n^t|^2 \frac{P_n^t}{N_0}\right)\right)^2} +Z_n^t 
\end{align*}
which \emph{does not depend} on $q_n^t$.
Next, let $A=\frac{V\lambda\ell |h_n^t|^2 \left(\log(2)\right)^2}{N_0 B Z_n^t}$ and $x=1+|h_n^t|^2 \frac{P_n^t}{N_0}$, then we have something in the form of
\begin{align*}
    A &= x \left(\log(x)\right)^2 \nonumber  = x \left(\log\left(\sfrac{1}{x}\right)\right)^2 .
\end{align*}
By dividing both sides by $1/4$, letting $x' = \sqrt{\frac{A}{4}} \frac{1}{\sqrt{x}}$, and rearranging, we have
\begin{align*}
    \sqrt{\sfrac{A}{4}}=x' e^{x'}
\end{align*}
that has a known solution of $ x' =W_k \left(\sqrt{\frac{A}{4}}\right)$ where $W_k(\cdot)$ is the Lambert $W$ function which solves $w\exp{w}=z$ for $w$.

To get the critical point for $P_n^t$, we unwrap and substitute $P_n^t = \frac{N_0}{|h_n^t|^2}(x-1)$, to get
\begin{align} \label{eqn:powerOpt2}
    P_n^\textnormal{t,opt} = \frac{N_0}{|h_n^t|^2} \left(\frac{A}{4} W_k\left(\sqrt{\frac{A}{4}}\right)^{-2}-1\right)
\end{align}
which has a single root at $k=0$ since $\sqrt{\frac{A}{4}}\geq 0$.